\documentclass[sigconf]{acmart} 

\usepackage{subfigure}
\usepackage{graphicx}
\usepackage{bbm}
\usepackage{multirow}
\usepackage{bbding}
\usepackage{amsmath}

\usepackage{algorithm}
\usepackage{algorithmic}
\usepackage{amsmath}

\usepackage{color}
\newcommand{\red}[1]{\textcolor{red}{#1}}
\newcommand{\green}[1]{\textcolor{green}{#1}}

\newtheorem{theorem}{Theorem}

\newtheorem{definition}{Definition}

\usepackage{booktabs}
\usepackage{threeparttable}

\usepackage{colortbl}
\definecolor{light-gray}{gray}{0.90}

\AtBeginDocument{%
  }

\copyrightyear{2023}
\acmYear{2023}
\setcopyright{acmlicensed}
\acmConference[KDD '23] {Proceedings of the 29th ACM SIGKDD Conference on Knowledge Discovery and Data Mining}{August 6--10, 2023}{Long Beach, CA, USA.}
\acmBooktitle{Proceedings of the 29th ACM SIGKDD Conference on Knowledge Discovery and Data Mining (KDD '23), August 6--10, 2023, Long Beach, CA, USA}
\acmPrice{15.00}
\acmISBN{979-8-4007-0103-0/23/08}
\acmDOI{10.1145/3580305.3599380}

\begin{document}
\title{HomoGCL: Rethinking Homophily in Graph Contrastive Learning}

\author{Wen-Zhi Li}
\affiliation{%
  \institution{CSE, Sun Yat-sen University, Guangzhou, China}
  \city{}\country{}
}
\affiliation{%
  \institution{AI Thrust, HKUST (GZ), Guangzhou, China}
  \city{}\country{}
}
\email{liwzh63@mail2.sysu.edu.cn}

\author{Chang-Dong Wang}
\authornote{Corresponding authors.}
\affiliation{%
  \institution{CSE, Sun Yat-sen University}
  \city{Guangzhou}\country{China}
}
\email{changdongwang@hotmail.com}

\author{Hui Xiong}
\authornotemark[1]
\affiliation{%
  \institution{AI Thrust, HKUST (GZ), Guangzhou, China}
  \city{}\country{}
}
\affiliation{%
  \institution{CSE, HKUST, Hong Kong, China}
  \city{}\country{}
}
\email{xionghui@ust.hk}

\author{Jian-Huang Lai}
\affiliation{%
  \institution{CSE, Sun Yat-sen University}
  \city{Guangzhou}\country{China}
}
\email{stsljh@mail.sysu.edu.cn}

\begin{abstract}
Contrastive learning (CL) has become the de-facto learning paradigm in self-supervised learning on graphs, which generally follows the ``augmenting-contrasting'' learning scheme. However, we observe that unlike CL in computer vision domain, CL in graph domain performs decently even \textit{without augmentation}. We conduct a systematic analysis of this phenomenon and argue that homophily, i.e., the principle that ``like attracts like'', plays a key role in the success of graph CL. Inspired to leverage this property explicitly, we propose HomoGCL, a model-agnostic framework to expand the positive set using neighbor nodes with neighbor-specific significances. Theoretically, HomoGCL introduces a stricter lower bound of the mutual information between raw node features and node embeddings in augmented views. Furthermore, HomoGCL can be combined with existing graph CL models in a plug-and-play way with light extra computational overhead. Extensive experiments demonstrate that HomoGCL yields multiple state-of-the-art results across six public datasets and consistently brings notable performance improvements when applied to various graph CL methods. Code is avilable at \url{https://github.com/wenzhilics/HomoGCL}.
\end{abstract}




\maketitle

\section{Introduction}

\begin{figure}[!t]
\centering
\subfigure[Vision datasets with SimCLR]
{
\centering
\includegraphics[width=0.47\linewidth]{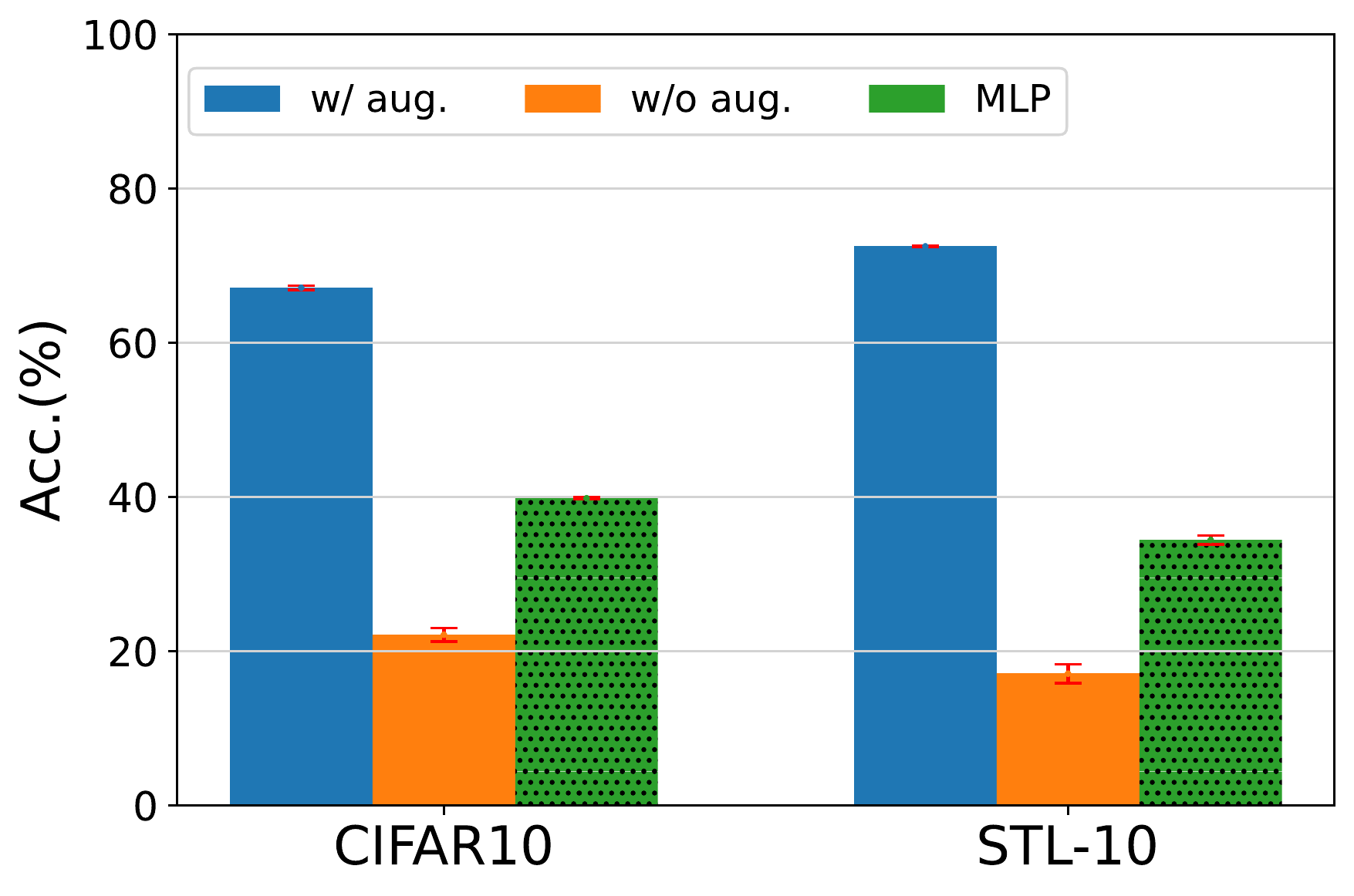}
}
\subfigure[Graph datasets with GRACE]
{
\centering
\includegraphics[width=0.47\linewidth]{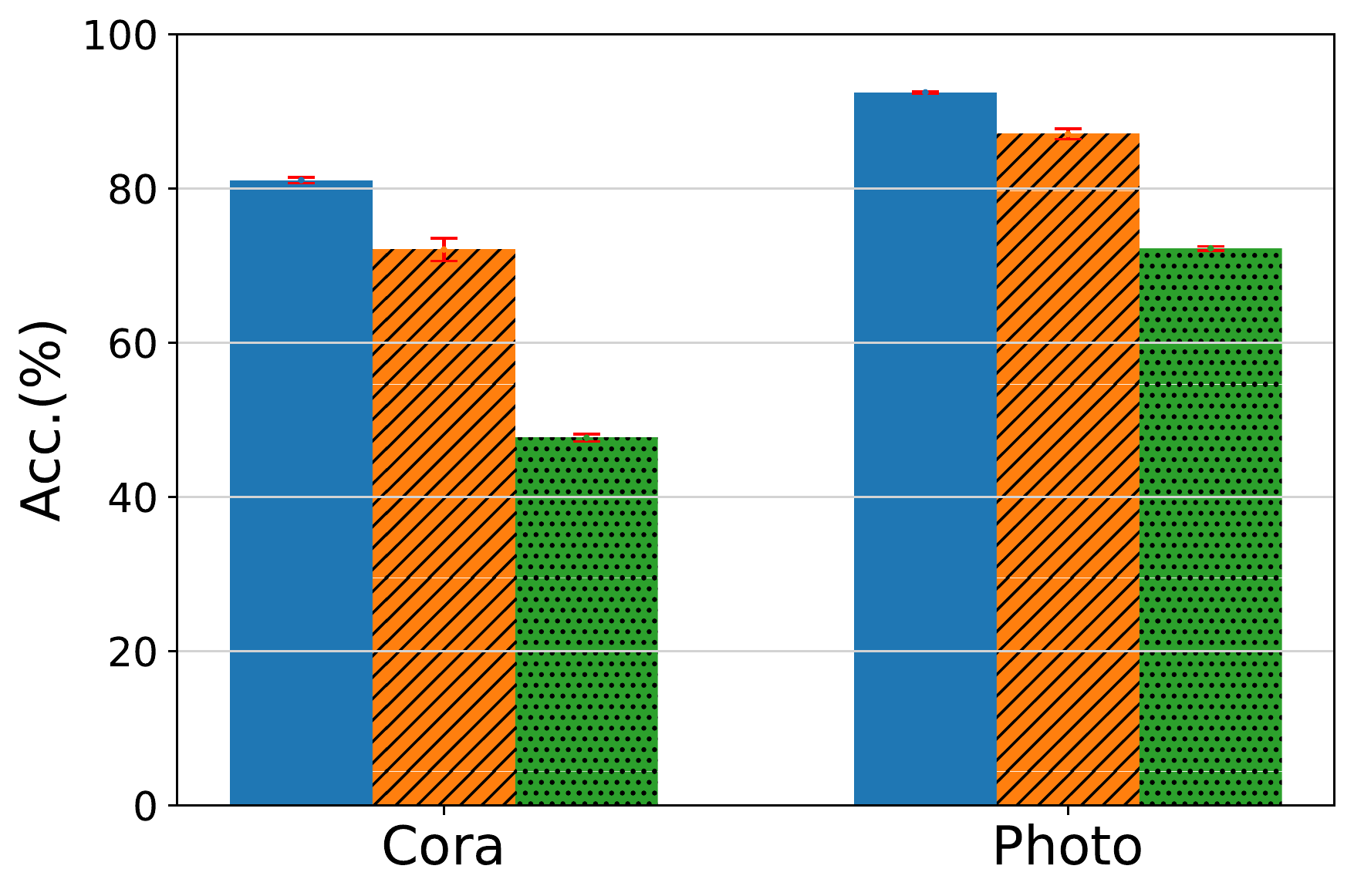}
}
\centering
\caption{Performance of CL in vision and graph domains with/without augmentation. SimCLR~\cite{simclr} and GRACE~\cite{grace}, two prevalent and similar CL architectures in vision and graph domains, are adopted on the respective datasets. MLP is the baseline by simply training a Multi-Layer Perceptron from image RGB features/raw node features. When without augmentation, the performance of vision datasets drops drastically, while the performance of graph datasets is rather stable and still outperforms the MLP counterpart.}
\label{fig:vcl_vs_gcl}
\vspace{-0.3cm}
\end{figure}

Graph Neural Networks (GNNs) have achieved overwhelming accomplishments on a variety of graph-based tasks like node classification and node clustering, to name a few~\cite{gcn, gat, graphsage, gnn1, gnn2}. They generally refer to the message passing mechanism where node features first propagate to neighbor nodes and then get aggregated to fuse the features in each layer.

Generally, GNNs are designed for supervised tasks which require adequate labeled data. However, it is hard to satisfy as annotated labels are always scarce in real-world scenarios~\cite{simclr, grace}. To tackle this common problem in deep learning, many pioneer endeavors have been made to self-supervised learning (SSL) in the computer vision domain, of which vision contrastive learning (VCL)~\cite{simclr, moco} has dominated the field. Generally, VCL follows the ``augmenting-contrasting'' learning pattern, in which the similarity between two augmentations of a sample (positive pair) is maximized, while the similarities between other samples (negative pairs) are minimized. The model can thus learn high-quality representations free of label notation. There are also many work adapting CL to graph representation learning, referred to as graph contrastive learning (GCL)~\cite{mvgrl, graphcl, gca}. Research hotspot in GCL mainly focuses on graph augmentation~\cite{simgrace, joao, autogcl, cgc, costa}, since unlike naturally rotating or cropping on images, graph augmentation would discard underlying semantic information which might result in undesirable performance. Though these elaborate graph augmentation-based approaches can achieve state-of-the-art performances on many graph-based tasks, we argue that the role of graph augmentation is still overemphasized. Empirically, we observe that GCL without augmentation can also achieve decent performance (Figure~\ref{fig:vcl_vs_gcl}(b)), which is quite different from VCL (Figure~\ref{fig:vcl_vs_gcl}(a)). A natural question arises thereby: 

\textit{What causes the huge gap between the performance declines of GCL and VCL when data augmentation is not leveraged?}

To answer this question, we conduct a systematic analysis and argue that \textit{homophily is the most important part of GCL.} Specifically, homophily is the phenomenon that ``like attracts like''~\cite{homo1}, or connected nodes tend to share the same label, which is a ubiquitous property in real-world graphs like citation networks or co-purchase networks~\cite{homo1, pitfall}. According to recent studies~\cite{geomgcn, homo2}, GNN backbones in GCL (such as GCN~\cite{gcn}, GAT~\cite{gat}, and GraphSAGE~\cite{graphsage}) heavily rely on the homophily assumption, as \textit{message passing} is applied for these models to aggregate information from direct neighbors for each node.

As a distinctive inductive bias of real-world graphs~\cite{geomgcn}, homophily is regarded as an appropriate guide in the case that node labels are not available~\cite{autossl}. Many recent GCLs have leveraged graph homophily implicitly by strengthening the relationship between connected nodes from different angles. For example, Xia et al.~\cite{progcl} tackle the false negative issue by avoiding similar neighbor nodes being negative samples, while Li et al.~\cite{gcool}, Wang et al.~\cite{clusterscl}, Park et al.~\cite{cgc}, and Lee et al.~\cite{afgcl} leverage community structure to enhance local connection. However, to the best of our knowledge, there is no such effort to directly leverage graph homophily, i.e., to treat neighbor nodes as positive.

In view of the argument, we are naturally inspired to leverage graph homophily explicitly. One intuitive approach is to simply treat neighbor nodes as positive samples indiscriminately. However, although connected nodes tend to share the same label in the homophily scenario, there also exist inter-class edges, especially near the decision boundary between two classes. Treating these inter-class connected nodes as positive (i.e., \textit{false positive}) would inevitably degenerate the overall performance. Therefore, our concern is distinguishing intra-class neighbors from inter-class ones and assigning more weights to them being positive. However, it is non-trivial since ground-truth node labels are agnostic in SSL. Thus, the main challenge is estimating the probability of neighbor nodes being positive in an unsupervised manner.

To tackle this problem, we devise HomoGCL, a model-agnostic method based on pair-wise similarity for the estimation. Specifically, HomoGCL leverages Gaussian Mixture Model (GMM) to obtain soft clustering assignments for each node, where node similarities are calculated as the indicator of the probability for neighbor nodes being \textit{true positive}. As a patch to augment positive pairs, HomoGCL is flexible to be combined with existing GCL approaches, including negative-sample-free ones like BGRL~\cite{bgrl} to yield better performance. 
Furthermore, theoretical analysis guarantees the performance boost over the base model, as the objective function in HomoGCL is a stricter lower bound of the mutual information between raw node features and augmented representations in augmented views.

We highlight the main contributions of this work as follows:
\begin{itemize}
\item We conduct a systematic analysis to study the mechanism of GCL. Our empirical study shows that graph homophily plays a key role in GCL, and many recent GCL models can be regarded as leveraging graph homophily implicitly.
\item We propose a novel GCL method, HomoGCL, to estimate the probability of neighbor nodes being positive, thus directly leveraging graph homophily.
Moreover, we theoretically show that HomoGCL introduces a stricter lower bound of mutual information between raw node features and augmented representations in augmented views.
\item Extensive experiments and in-depth analysis demonstrate that HomoGCL outperforms state-of-the-art GCL models across six public benchmark datasets. Furthermore, HomoGCL can consistently yield performance improvements when applied to various GCL methods in a plug-and-play manner.
\end{itemize}

The rest of this paper is organized as follows. In Section~\ref{sec:related}, we briefly review related work. In Section~\ref{sec:meth}, we first provide the basic preliminaries of graph contrastive learning. Then, we conduct an empirical study to delve into the graph homophily in GCL, after which we propose the HomoGCL model to directly leverage graph homophily. Theoretical analysis and complexity analysis are also provided. In Section~\ref{sec:exp}, we present the experimental results and in-depth analysis of the proposed model. Finally, we conclude this paper in Section~\ref{sec:conc}.

\section{Related Work}\label{sec:related}

In this section, we first review pioneer work on graph contrastive learning. Then, we review graph homophily, which is believed to be a useful inductive bias for graph data in-the-wild~\cite{homo1}.

\subsection{Graph Contrastive Learning}

GCL has gained popularity in the graph SSL community for its expressivity and simplicity~\cite{graphcl, gcc, bestpractice}. It generally refers to the paradigm of making pair-view representations to agree with each other under proper data augmentations. Among them, DGI~\cite{dgi}, HDI~\cite{hdmi}, GMI~\cite{gmi}, and InfoGCL~\cite{infogcl} directly measure mutual information between different views. MVGRL~\cite{mvgrl} maximizes information between the cross-view representations of nodes and graphs. GRACE~\cite{grace} and its variants GCA~\cite{gca}, ProGCL~\cite{progcl}, ARIEL~\cite{ariel}, gCooL~\cite{gcool} adopt SimCLR~\cite{simclr} framework for \textit{node-level} representations, while for graph-level representations, GraphCL~\cite{graphcl}, JOAO~\cite{joao}, SimGRACE~\cite{simgrace} also adopt the SimCLR framework. Additionally, G-BT~\cite{gbt}, BGRL~\cite{bgrl}, AFGRL~\cite{afgcl}, and CCA-SSG~\cite{cca} adopt new CL frameworks that free them from negative samples or even data augmentations. In this paper, we propose a method to expand \textit{positive samples} in GCL on \textit{node-level} representation learning, which can be combined with existing node-level GCL in a plug-and-play way.

\begin{figure*}[!t]
\centering
\subfigure[Similarity histogram on CIFAR10]
{
\centering
\includegraphics[width=0.32\linewidth]{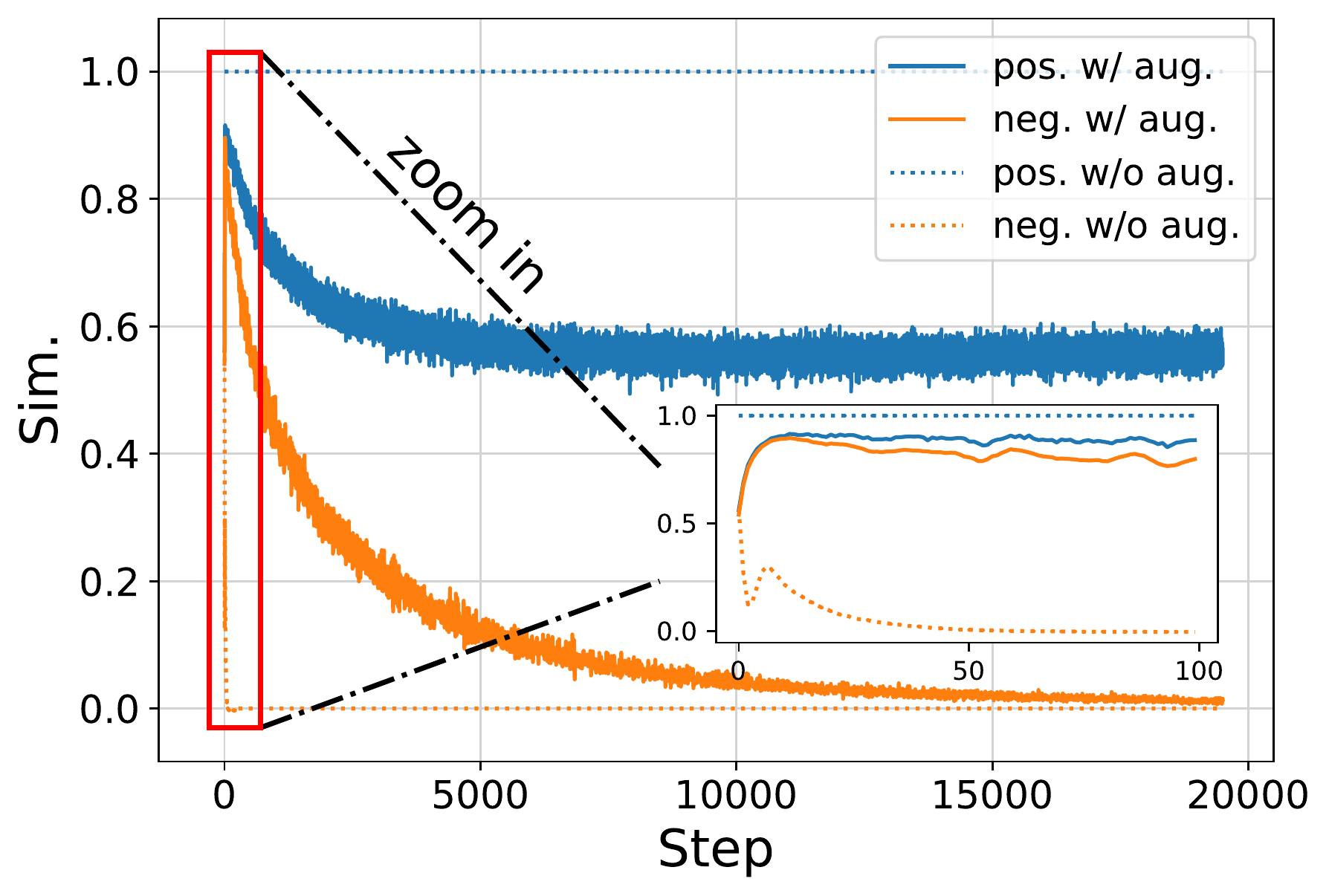}
}
\subfigure[Similarity histogram on Cora]
{
\centering
\includegraphics[width=0.31\linewidth]{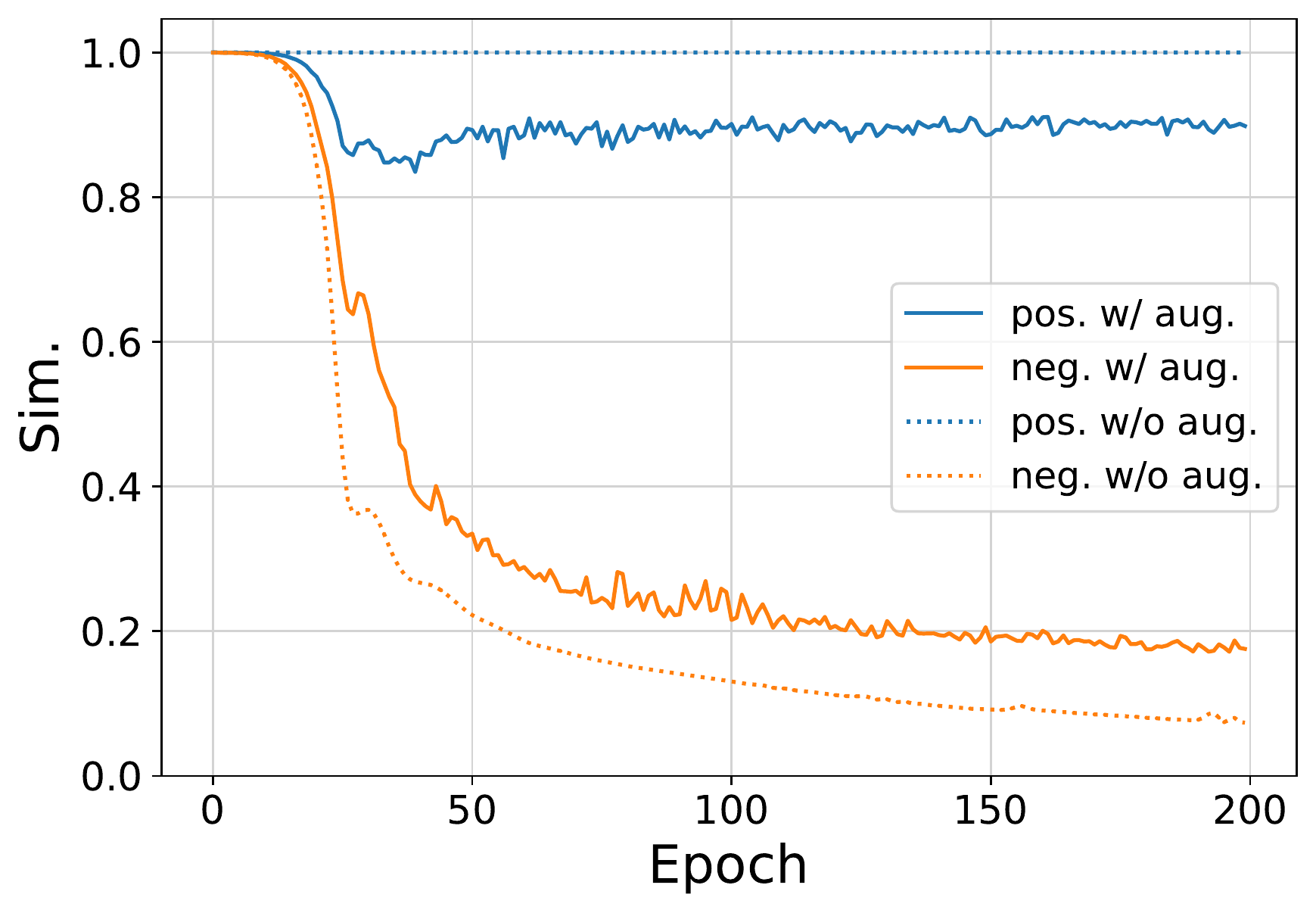}
}
\subfigure[Ablation study on two graph datasets]
{
\centering
\includegraphics[width=0.33\linewidth]{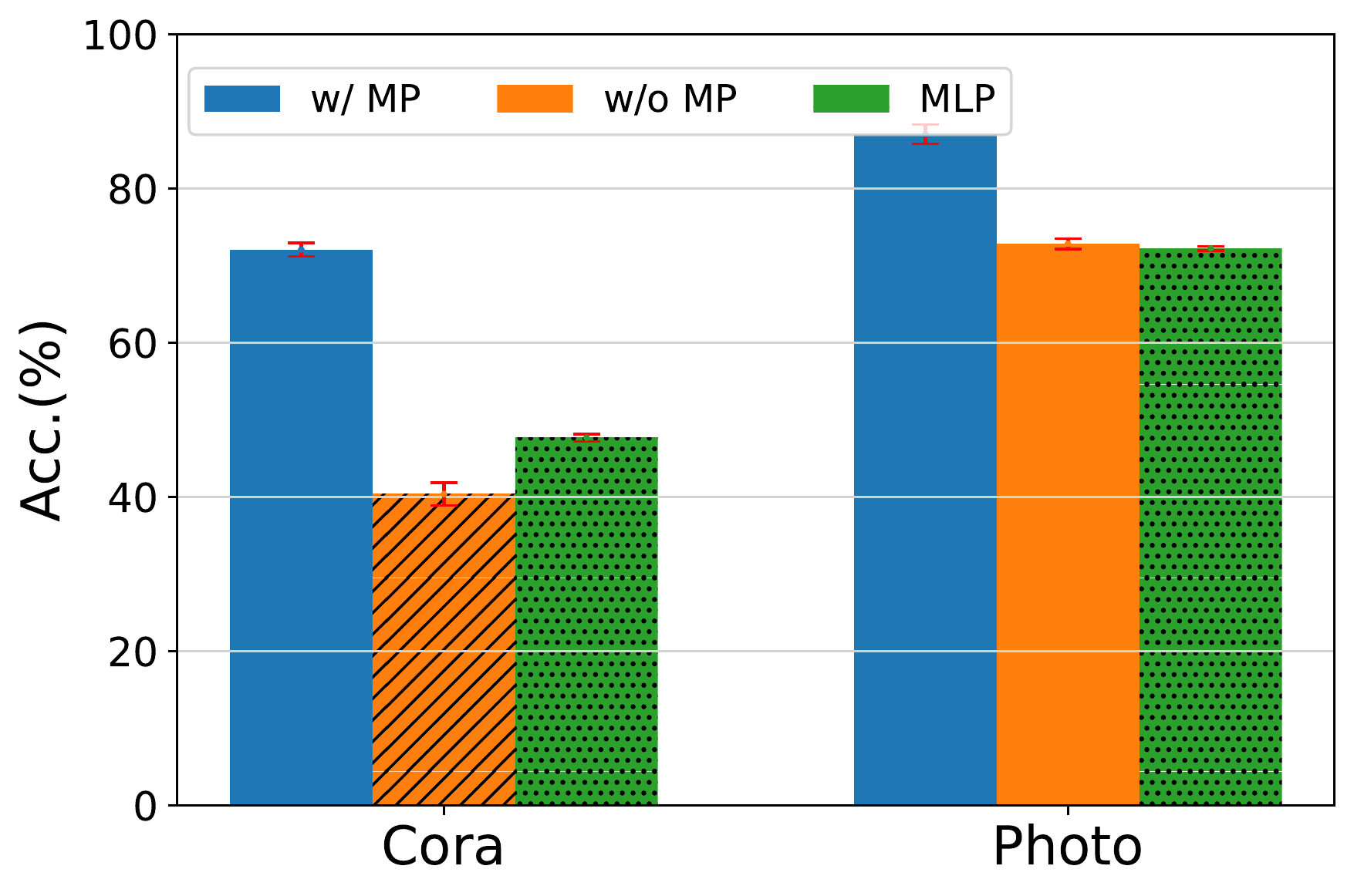}
}
\centering
\caption{Empirical studies on graph homophily. (a), (b) are similarities between positive and negative pairs w.r.t. the training processes with/without augmentation on vision dataset CIFAR10 and graph dataset Cora. The similarity between negative pairs drops to 0 swiftly on CIFAR10 without augmentation, while the similarity between negative pairs drops gradually on Cora without augmentation, which is analogous to its counterpart with augmentation. Please note that the similarity between positive pairs remains as 1 when without augmentation since the two views are identical. To analyze the role that homophily played in this phenomenon, we conduct an ablation study for GRACE without augmentation in (c) by (1) only enabling message passing (w/ MP), and (2) disabling message passing (w/o MP), together with the MLP baseline on two graph datasets Cora and Photo. The performance shows the functionality of message passing, which relies on the homophily assumption.
}\label{fig:emp}
\end{figure*}

\subsection{Graph Homophily}

Graph homophily, i.e., ``birds of a feather flock together''~\cite{homo1}, indicates that connected nodes often belong to the same class, which is useful prior knowledge in real-world graphs like citation networks, co-purchase networks, or friendship networks~\cite{pitfall}. A recent study AutoSSL~\cite{autossl} shows graph homophily is an effective guide in searching the weights on various self-supervised pretext tasks, which reveals the effectiveness of homophily in graph SSL. However, there is no such effort to leverage homophily directly in GCL to the best of our knowledge. It is worth noting that albeit heterophily graphs also exist~\cite{geomgcn} where dissimilar nodes tend to be connected, the related study is still at an early stage even in the supervised setting~\cite{linkx}. Therefore, to be consistent with previous work on GCL, we only consider the most common homophily graphs in this work.

\section{Methodology}\label{sec:meth}

In this section, we first introduce the preliminaries and notations about GCL. We then conduct a systematic investigation on the functionality of homophily in GCL. Finally, we propose the HomoGCL method to leverage graph homophily directly in GCL.

\subsection{Preliminaries and Notations}

Let $\mathcal{G=\{V,E \}}$ be a graph, where $\mathcal{V}=\{v_1, v_2,\cdots v_N\}$ is the node set with $N$ nodes and $\mathcal{E} \subseteq \mathcal{V} \times \mathcal{V}$ is the edge set. The adjacency matrix and the feature matrix are denoted as $\mathbf{A} \in \{ 0,1 \}^{N \times N}$ and $\mathbf{X} \in \mathbb{R}^{N \times d}$, respectively, where $\mathbf{A}_{ij}=1$ iff $(v_i,v_j)\in\mathcal{E}$, and $\boldsymbol{x}_i\in\mathbb{R}^d$ is the $d$-dim raw feature of node $v_i$.
The main notions used throughout the paper are summarized in Appendix~\ref{app:notation}.

Given a graph $\mathcal{G}$ with no labels, the goal of GCL is to train a graph encoder $f_{\mathbf{\Theta}}(\mathbf{X},\mathbf{A})$ and get node embeddings that can be directly applied to downstream tasks like node classification and node clustering. Take one of the most popular GCL framework GRACE~\cite{grace} as an example, two augmentation functions $t_1$, $t_2$ (typically randomly dropping edges and masking features) are firstly applied to the graph $\mathcal{G}$ to generate two graph views $\mathcal{G}_1=(\mathbf{X}_1,\mathbf{A}_1)=t_1(\mathcal{G})$ and $\mathcal{G}_2=(\mathbf{X}_2,\mathbf{A}_2)=t_2(\mathcal{G})$. Then, the two augmented graphs are encoded by the same GNN encoder, after which we get node embeddings $\mathbf{U}=f_{\mathbf{\Theta}}(\mathbf{X}_1,\mathbf{A}_1)$ and $\mathbf{V}=f_{\mathbf{\Theta}}(\mathbf{X}_2,\mathbf{A}_2)$. Finally, the loss function is defined by the InfoNCE~\cite{infonce} loss as
\begin{equation}\label{eq:grace}
\mathcal{L}=\frac{1}{2 N} \sum_{i=1}^{N}\left(\ell\left(\boldsymbol{u}_{i}, \boldsymbol{v}_{i}\right)+\ell\left(\boldsymbol{v}_{i}, \boldsymbol{u}_{i}\right)\right),
\vspace{-0.3cm}
\end{equation}
with
\begin{equation}\label{eq:fine_grain_grace}
\begin{aligned}
&\ell\left(\boldsymbol{u}_{i}, \boldsymbol{v}_{i}\right)=\\
&\log \frac{e^{\theta(\boldsymbol{u}_{i}, \boldsymbol{v}_{i}) / \tau}}{\underbrace{e^{\theta(\boldsymbol{u}_{i}, \boldsymbol{v}_{i}) / \tau}}_{\text{positive pair}}+\underbrace{\sum\limits_{j \neq i} e^{\theta(\boldsymbol{u}_{i}, \boldsymbol{v}_{j}) / \tau}}_{\text{inter-view negative pairs}}+\underbrace{\sum\limits_{j \neq i} e^{\theta(\boldsymbol{u}_{i}, \boldsymbol{u}_{j}) / \tau}}_{\text{intra-view negative pairs}}},
\end{aligned}
\end{equation}
where $\theta(\cdot,\cdot)$ is the similarity function and $\tau$ is a temperature parameter.
In principle, any GNN can be served as the graph encoder. Following recent work~\cite{gca, progcl, ariel}, we adopt a two-layer graph convolutional network (GCN)~\cite{gcn} as the encoder $f_{\mathbf{\Theta}}$ by default, which can be formalized as
\begin{equation}
f_{\mathbf{\Theta}}(\mathbf{X},\mathbf{A})=\mathbf{H}^{(2)}=\mathbf{\hat{A}}\sigma(\mathbf{\hat{A}}\mathbf{X}\mathbf{W}^{(1)})\mathbf{W}^{(2)},
\end{equation}
where $\mathbf{\hat{A}}=\mathbf{\tilde{D}}^{-1/2}(\mathbf{A}+\mathbf{I}_N)\mathbf{\tilde{D}}^{-1/2}$ with $\mathbf{\tilde{D}}$ being the degree matrix of $\mathbf{A}+\mathbf{I}_N$ and $\mathbf{I}_N$ being the identity matrix, $\mathbf{W}$ are learnable weight matrices, and $\sigma(\cdot)$ is the $ReLU$ activation function~\cite{relu}.

\begin{figure*}[!t]
\centering
\subfigure[Workflow of HomoGCL]
{
\centering
\includegraphics[width=0.78\linewidth]{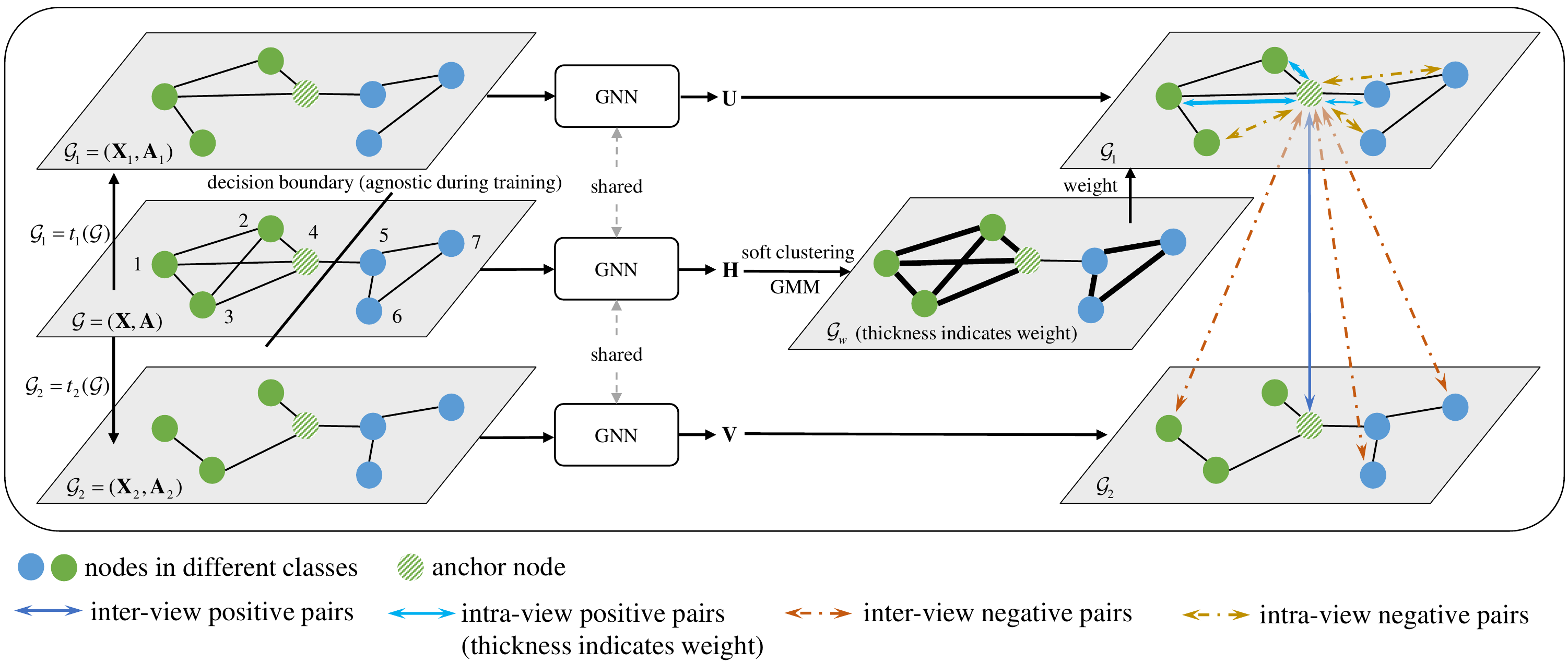}
}
\subfigure[Soft clustering]
{
\centering
\includegraphics[width=0.19\linewidth]{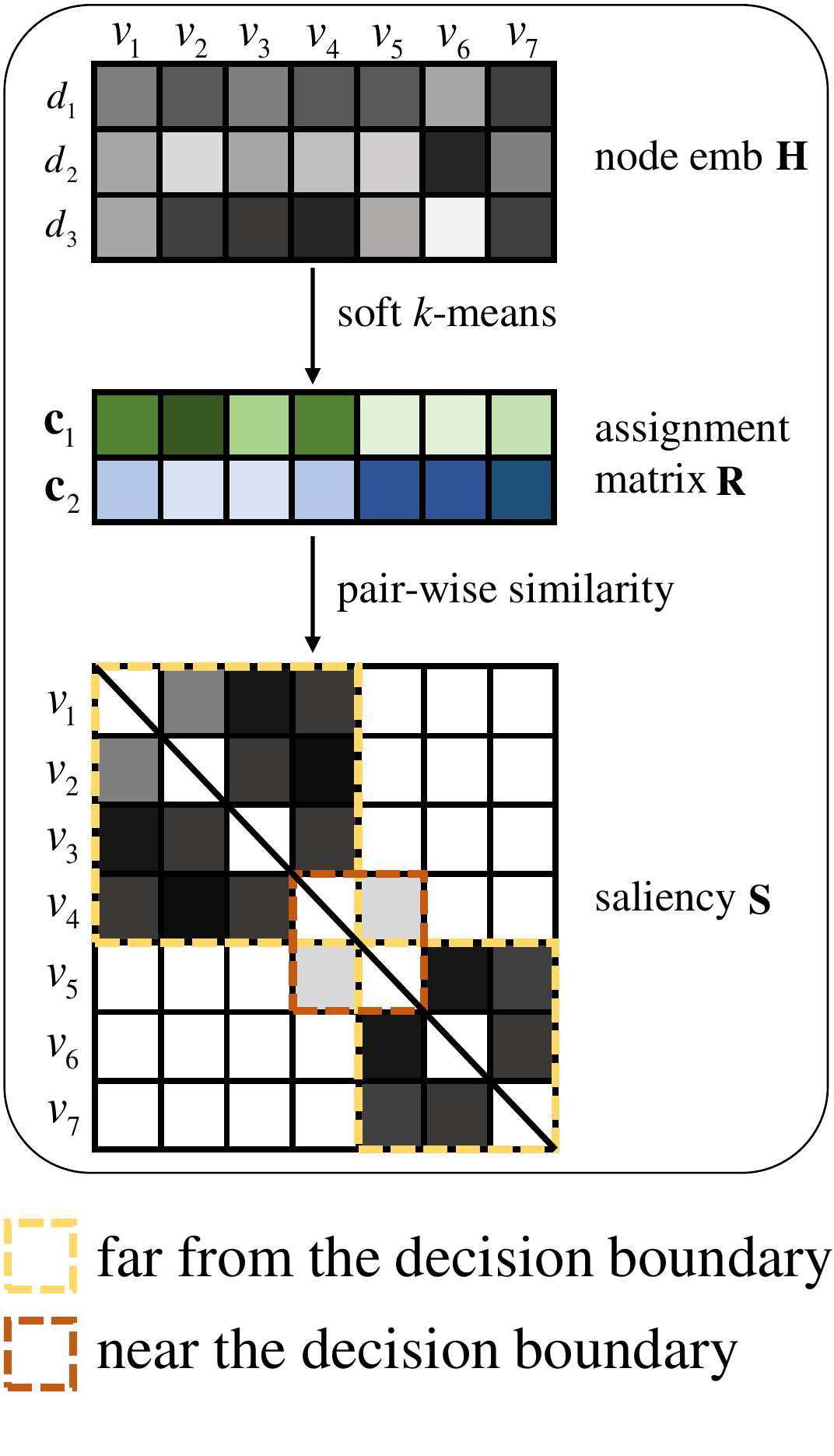}
}
\centering
\caption{The pipeline of HomoGCL (a). Two graph views $\mathcal{G}_1$ and $\mathcal{G}_2$ are first generated via graph augmentation from graph $\mathcal{G}$, after which the three graphs are fed into the shared GNN encoder to learn representations. The representation of $\mathcal{G}$ is utilized to generate soft clustering assignments via Gaussian Mixture Model, based on which the edge saliency is calculated (b). Edge saliency is leveraged as the weight of neighbor nodes being positive.}\label{fig:flowchart}
\end{figure*}

\subsection{Homophily in GCL: an Empirical Investigation}

In Figure~\ref{fig:vcl_vs_gcl}, we can observe that the performance of VCL collapses without data augmentation, while the performance of GCL is still better than the MLP counterpart, which implies that there is a great discrepancy in the mechanism between VCL and GCL, although they adopt seemingly similar frameworks SimCLR~\cite{simclr} and GRACE~\cite{grace}. To probe into this phenomenon, we plot the similarities between positive and negative samples w.r.t. the training processes, as shown in Figure~\ref{fig:emp} (a), (b). 

From the figures, we can see that for vision dataset CIFAR10, the similarity between negative pairs drops swiftly without augmentation. We attribute this to that the learning objective is too trivial by enlarging the similarity between the sample and \textit{itself} while reducing the similarity between the sample and any other sample in the batch. Finally, the model can only learn one-hot-like embeddings, which eventually results in the poor performance. However, for graph dataset Cora, the similarity between negative pairs without augmentation still drops gradually and consistently, which is analogous to its counterpart with augmentation. We attribute this phenomenon to the message passing in GNN. Albeit InfoNCE loss also enlarges the similarity between the identical samples in two views while reducing the similarity between two different samples, the message passing in GNN enables each node to aggregate information from its neighbors, which leverages graph homophily implicitly to avoid too trivial discrimination. From another perspective, the message passing in GNN enables each node to no longer be independent of its neighbors. As a result, only the similarities between two relatively far nodes (e.g., neighbors $\geq 3$-hop for a two-layer GCN), are reduced.

\begin{definition}
(Homophily). The homophily in a graph $\mathcal{G}$ is defined as the fraction of intra-class edges. Formally, with node label $\mathbf{Y}$, homophily is defined as
\begin{equation}
h(\mathcal{G}, \mathbf{Y})=\frac{1}{|\mathcal{E}|} \sum_{\left(v_{1}, v_{2}\right) \in \mathcal{E}} \mathbbm{1}\left(\boldsymbol{y}_1=\boldsymbol{y}_2\right),
\end{equation}
where $\boldsymbol{y}_i$ denotes the label of $v_i$ and $\mathbbm{1}(\cdot)$ is the indication function.
\end{definition}

As analyzed above, we hypothesize that message passing which relies on the homophily assumption prevents GCL without augmentation from corruption. To validate this, we conduct another ablation experiment on two graph datasets Cora and Photo, as shown in Figure~\ref{fig:emp} (c).

\subsubsection*{\textbf{Experimental setup}} 
We devise two variants of GRACE~\cite{grace} without augmentation, i.e., GRACE w/ MP and GRACE w/o MP, where the former is the version that message passing is leveraged (i.e., the w/o aug. version in Figure~\ref{fig:vcl_vs_gcl}(b)), while the latter is the version that message passing is blocked via forcing each node to only connect with itself (i.e., substituting the adjacency matrix with the eye matrix). For MLP, we directly train an MLP with raw node features as the input in a supervised learning manner without graph structure, which is the same as in Figure~\ref{fig:vcl_vs_gcl}. The two variants are equipped with the same 2-layer GCN backbone of 256 dim hidden embeddings and are trained until convergence.

\subsubsection*{\textbf{Observations}} 
From Figure~\ref{fig:emp} (c), we can see that the performance of GRACE (w/o MP) is only on par with or even worse than the MLP counterpart, while GRACE (w/ MP) outperforms both by a large margin on both datasets. It meets our expectation that by disabling message passing, nodes in GRACE (w/o MP) cannot propagate features to their neighbors, which degenerates them to a similar situation of VCL without augmentation. GRACE (w/ MP), on the other hand, can still maintain the performance even without raw features. In a nutshell, this experiment verifies our hypothesis that message passing, which relies on the homophily assumption, is the key factor of GCL.

\subsection{HomoGCL}

As analyzed above, graph homophily plays a crucial role in the overall performance of GCL. Therefore, we are naturally inspired to leverage this property explicitly. Simply assigning neighbor nodes as positive is non-ideal, as nodes near the decision boundaries tend to link with nodes from another class, thus being \textit{false positive}. To mitigate such effect, it is expected to estimate the probability of neighbor nodes being \textit{true positive}. However, the task is intractable as node labels are not available to identify the boundaries in the SSL setting. To tackle this challenge, we leverage GMM on $k$-means hard clusters to get soft clustering assignments of the initial graph $\mathcal{G}$, where pair-wise similarity (saliency) is calculated as the aforementioned probability. The overall framework of HomoGCL and soft clustering are illustrated in Figure~\ref{fig:flowchart}. Please note that although we introduce HomoGCL based on GRACE~\cite{grace} framework as an example, HomoGCL is framework-agnostic and can be combined with other GCL frameworks.

\subsubsection*{\textbf{Soft clustering for pair-wise node similarity}}
An unsupervised method like clustering is needed to estimate the probability of neighbor nodes being true positive. However, traditional clustering methods like $k$-means can only assign a hard label for each node, which cannot satisfy our needs. To tackle this problem, we view $k$-means as a special case of GMM\cite{prml, autossl}, where soft clustering is made possibile based on the posterior probabilities. Specifically, for GMM with $k$ centroids $\{\boldsymbol{c}_1, \boldsymbol{c}_2, \cdots, \boldsymbol{c}_k\}$ defined by the mean embeddings of nodes in different $k$-means hard labels, posterior probability can be calculated as
\begin{equation}
p\left(\boldsymbol{h}_i \mid \boldsymbol{c}_{j}\right)=\frac{1}{\sqrt{2 \pi \sigma^{2}}} \exp \left(-\frac{\left\|\boldsymbol{h}_i-\boldsymbol{c}_{j}\right\|_{2}}{2 \sigma^{2}}\right),
\end{equation}
where $\sigma^2$ is the variance of Gaussian distribution. By considering an equal prior $p(\boldsymbol{c}_1)=p(\boldsymbol{c}_2)=\cdots=p(\boldsymbol{c}_k)$, the probability of node feature $\boldsymbol{h}_i$ belonging to cluster $\boldsymbol{c}_j$ can be calculated by the Bayes rule as
\begin{equation}\label{eq:gmm}
\begin{aligned}
p\left(\boldsymbol{c}_{j} \mid \boldsymbol{h}_i\right)=\frac{p\left(\boldsymbol{c}_{j}\right) p\left(\boldsymbol{h}_i \mid \boldsymbol{c}_{j}\right)}{\sum\limits_{r=1}^{k} p\left(\boldsymbol{c}_{r}\right) p\left(\boldsymbol{h}_i \mid \boldsymbol{c}_{r}\right)}=\frac{\exp \left(-\frac{\left(\boldsymbol{h}_i-\boldsymbol{c}_{j}\right)^{2}}{2 \sigma^{2}}\right)}{\sum\limits_{r=1}^{k} \exp \left(-\frac{\left(\boldsymbol{h}_i-\boldsymbol{c}_{r}\right)^{2}}{2 \sigma^{2}}\right)}.
\end{aligned}
\end{equation}
In this way, we can get a cluster assignment matrix $\mathbf{R}\in\mathbb{R}^{N\times k}$ where $\mathbf{R}_{ij}=p(\boldsymbol{c}_{j} \mid \boldsymbol{h}_i)$ indicates the soft clustering value between node $v_i$ and cluster $\boldsymbol{c}_j$.

Based on the assignment matrix $\mathbf{R}$, we are able to calculate a node saliency $\mathbf{S}_{ij}$ between any connected node pair ($v_i$, $v_j$) via $\mathbf{S}_{ij}=\mathrm{norm}(\mathbf{R}_{i})\cdot \mathrm{norm}(\mathbf{R}_j^\top)$ with $\mathrm{norm}(\cdot)$ being the $L_2$ normalization on the cluster dimension. $\mathbf{S}_{ij}$ can thus indicate the connection intensity between node $v_i$ and $v_j$, which is an estimated probability of neighbors being true positive.

\subsubsection*{\textbf{Loss function}}
As the probability of neighbors being positive, node saliency $\mathbf{S}$ can be utilized to expand positive samples in both views. Specifically, Eq.~\eqref{eq:fine_grain_grace} is converted to
\begin{equation}\label{eq:infonceloss}
\ell_{cont}(\boldsymbol{u}_i,\boldsymbol{v}_i)=\log \frac{\mathrm{pos}}{\mathrm{pos}+\mathrm{neg}}
\end{equation}
with
\begin{align}
\mathrm{pos}&=\underbrace{e^{\theta(\boldsymbol{u}_i,\boldsymbol{v}_i)/ \tau}}_{\text{inter-view positive pair}}+\underbrace{\sum_{j\in \mathcal{N}_{\boldsymbol{u}}(i)}e^{\theta(\boldsymbol{u}_i,\boldsymbol{u}_j)/\tau}\cdot \mathbf{S}_{ij}}_{\text{intra-view positive pairs}},\label{eq:pos}\\
\mathrm{neg}&=\underbrace{\sum\limits_{j \notin \{i\cup\mathcal{N}_{\boldsymbol{v}}(i)\}}e^{\theta(\boldsymbol{u}_i,\boldsymbol{v}_j)/ \tau}}_{\text{inter-view negative pairs}} + \underbrace{\sum\limits_{j \notin \{i\cup\mathcal{N}_{\boldsymbol{u}}(i)\}}e^{\theta(\boldsymbol{u}_i,\boldsymbol{u}_j)/ \tau}}_{\text{intra-view negative pairs}},\label{eq:neg}
\end{align}
where $\mathcal{N}_{\boldsymbol{u}}(i)$, $\mathcal{N}_{\boldsymbol{v}}(i)$ are the neighbor sets of node $v_i$ in two views. The contrastive loss is thus defined as
\begin{equation}\label{eq:cont}
\mathcal{L}_{cont}=\frac{1}{2 N} \sum_{i=1}^{N}\left(\ell_{cont}\left(\boldsymbol{u}_{i}, \boldsymbol{v}_{i}\right)+\ell_{cont}\left(\boldsymbol{v}_{i}, \boldsymbol{u}_{i}\right)\right).
\end{equation}
In addition to the contrastive loss, we also leverage the homophily loss~\cite{autossl} explicitly via
\begin{equation}\label{eq:homoloss}
\begin{aligned}
\mathcal{L}_{homo}=\frac{1}{k|\mathcal{E}|} \sum_{r=1}^{k} \sum_{\left(v_{i}, v_{j}\right) \in \mathcal{E}} \mathrm{MSE}\left(p\left(\boldsymbol{c}_{r} \mid \boldsymbol{h}_{i}\right), p\left(\boldsymbol{c}_{r} \mid \boldsymbol{h}_{j}\right)\right),
\end{aligned}
\end{equation}
where $\mathrm{MSE}(\cdot)$ is the Mean Square Error. The contrastive loss and the homophily loss are combined in a multi-task learning manner with coefficient $\alpha$ as
\begin{equation}\label{eq:loss}
\begin{aligned}
\mathcal{J}=\mathcal{L}_{cont}+\alpha \mathcal{L}_{homo}.
\end{aligned}
\end{equation}

It is noteworthy that since HomoGCL is a way to expand positive samples, it can be combined with many node-level GCLs — even negative-free ones like BGRL~\cite{bgrl} — via the saliency $\mathbf{S}$ and the homophily loss Eq.~\eqref{eq:homoloss}, which will be discussed in Section~\ref{sec:plusbgrl}.

\subsection{Theoretical Analysis}

Though simple and intuitive by design, the proposed HomoGCL framework is theoretically guaranteed to boost the performance of base models from the Mutual Information (MI) maximization perspective, as induced in Theorem~\ref{theorem:1} with GRACE as an example.
\begin{theorem}
The newly proposed contrastive loss $\mathcal{L}_{cont}$ in Eq.~\eqref{eq:cont} is a stricter lower bound of MI between raw node features $\mathbf{X}$ and node embeddings $\mathbf{U}$ and $\mathbf{V}$ in two augmented views, comparing with the raw contrastive loss $\mathcal{L}$ in Eq.~\eqref{eq:grace} proposed by GRACE. Formally, 
\begin{equation}
\mathcal{L} \leq \mathcal{L}_{cont} \leq I(\mathbf{X};\mathbf{U},\mathbf{V}).
\end{equation}
\label{theorem:1}
\vspace{-0.2cm}
\end{theorem}

\begin{proof}
See Appendix~\ref{app:proof}. 
\end{proof}

From Theorem~\ref{theorem:1}, we can see that maximizing $\mathcal{L}_{cont}$ is equivalent to maximizing a lower bound of the mutual information between raw node features and learned node representations, which guarantees model convergence~\cite{gca, theorem1, infonce2, theorem2, theorem3}. Furthermore, the lower bound derived by HomoGCL is stricter than that of the GRACE, which provides a theoretical basis for the performance boost of HomoGCL over the base model.

\begin{algorithm}[!tb]
\caption{HomoGCL (based on GRACE)}
\label{algorithm}
\flushleft{\textbf{Input}: $\mathcal{G}=(\mathbf{X},\mathbf{A})$}\\
\textbf{Parameters}: Number of clusters $k$, coefficient $\alpha$, GRACE-related hyperparameters \\
\textbf{Output}: trained GNN encoder $f_{\Theta}$ and node embeddings $\mathbf{H}$ \\

\begin{algorithmic}[1] 
\STATE Initialize GNN encoder $f_{\Theta}$
\WHILE{not converge}
\STATE Generate two augmentation functions $t_1$ and $t_2$
\STATE Generate two augmented graphs via $\mathcal{G}_1=t_1(\mathcal{G})$ and $\mathcal{G}_2=t_2(\mathcal{G})$
\STATE Obtain node embeddings $\mathbf{H},\mathbf{U},\mathbf{V}$ of $\mathcal{G},\mathcal{G}_1,\mathcal{G}_2$ using the same encoder $f_{\Theta}$
\STATE Perform $k$-means clustering on $\mathbf{H}$ and obtain centroids $\{\boldsymbol{c}_1,\boldsymbol{c}_2,\cdots,\boldsymbol{c}_k\}$
\STATE Calculate $p(\boldsymbol{c}_j \mid \boldsymbol{h}_i)$ for each node-centroid pair according to Eq.~\eqref{eq:gmm} and get the assignment matrix $\mathbf{R}$
\STATE Calculate saliency $\mathbf{S}$ for connected node pairs
\STATE Compute the contrastive loss $\mathcal{L}_{cont}$ via Eq.~\eqref{eq:infonceloss}~\eqref{eq:pos}~\eqref{eq:neg}~\eqref{eq:cont}
\STATE Compute the homophily loss $\mathcal{L}_{homo}$ via Eq.~\eqref{eq:homoloss}
\STATE Compute the final loss $\mathcal{J}$ via Eq.~\eqref{eq:loss}
\STATE Update the parameters of $f_{\Theta}$ via $\mathcal{J}$
\ENDWHILE
\STATE \textbf{return} $f_{\Theta}$, $\mathbf{H}$
\end{algorithmic}
\end{algorithm}

\subsection{Complexity Analysis}\label{sec:complexity}

The overview of the training algorithm of HomoGCL (based on GRACE) is elaborated in Algorithm~\ref{algorithm}, based on which we analyze its time and space complexity. It is worth mentioning that the extra calculation of HomoGCL introduces light computational overhead over the base model.

\subsubsection*{\textbf{Time Complexity}}
We analyze the time complexity according to the pseudocode. For line 6, the time complexity of $k$-means with $t$ iterations, $k$ clusters, $N$ node samples with $d^{\prime}$-dim hidden embeddings is $\mathcal{O}(tkNd^{\prime})$.  Obtaining $k$ cluster centroids needs $\mathcal{O}(Nd^{\prime})$ based on the hard pseudo-labels obtained by $k$-means. For line 7, we need to calculate the distance between each node and each cluster centroid, which is another $\mathcal{O}(kNd^{\prime})$ overhead to get the assignment matrix $\mathbf{R}$. For line 8, the saliency $\mathbf{S}$ can be obtained via $L_2$ norm and vector multiplication for connected nodes in $\mathcal{O}(k(N+|\mathcal{E}|))$. For line 10, the homophily loss can be calculated in $\mathcal{O}(k|\mathcal{E}|)$ by definition. Overall, the extra computational overhead over the base model is $\mathcal{O}(tkNd^{\prime}+k(N+|\mathcal{E}|))$ for HomoGCL, which is lightweight compared with the base model, as $k$ is usually set to a small number.

\subsubsection*{\textbf{Space Complexity}}
Extra space complexity over the base model for HomoGCL is introduced by the $k$-means algorithm, the assignment matrix $\mathbf{R}$, and the saliency $\mathbf{S}$. For the $k$-means algorithm mentioned above, its space complexity is $\mathcal{O}(Nd^{\prime})$. For $\mathbf{R}\in\mathbb{R}^{N\times k}$, its space complexity is $\mathcal{O}(kN)$. As only the saliency between each connected node pair will be leveraged, the saliency costs $\mathcal{O}(|\mathcal{E}|)$. Overall, the extra space complexity of HomoGCL is $\mathcal{O}((d^{\prime}+k)N+|\mathcal{E}|)$, which is lightweight and on par with the GNN encoder.

\section{Experiments}\label{sec:exp}

In this section, we evaluate the effectiveness of HomoGCL by answering the following research questions:
\begin{itemize}
\item[\textbf{RQ1:}] Does HomoGCL outperform existing baseline methods on node classification and node clustering? Can it consistently boost the performance of prevalent GCL frameworks?
\item[\textbf{RQ2:}] Can the saliency $\mathbf{S}$ distinguish the importance of neighbor nodes being positive?
\item[\textbf{RQ3:}] Is HomoGCL sensitive to hyperparameters?
\item[\textbf{RQ4:}] How to intuitively understand the representation capability of HomoGCL over the base model? 
\end{itemize}

\begin{table}[!t]
\centering
\begin{center}
\caption{Statistics of datasets used in the paper.}\label{table:datasets}
\vspace{-0.2cm}
\scalebox{0.95}{
\begin{tabular}{c|cccc}
\toprule
{Dataset} & {\#Nodes} & {\#Edges} & {\#Features} & {\#Classes} \\
\midrule
Cora             & 2,708        & 10,556          & 1,433               & 7 \\
CiteSeer         & 3,327        & 9,228          & 3,703               & 6 \\
PubMed           & 19,717       & 88,651         & 500                 & 3 \\
Photo            & 7,650        & 238,163        & 745                 & 8 \\
Computer         & 13,752       & 491,722        & 767                 & 10 \\
arXiv            & 169,343      & 1,166,243      & 128                 & 40 \\
\bottomrule
\end{tabular}
}
\end{center}
\vspace{-0.3cm}
\end{table}

\subsection{Experimental Setup}

\subsubsection*{\textbf{Datasets}}
We adopt six publicly available real-world benchmark datasets, including three citation networks Cora, CiteSeer, PubMed~\cite{cora}, two co-purchase networks Amazon-Photo (Photo), Amazon-Computers (Computer)~\cite{pitfall}, and one large-scale network ogbn-arXiv (arXiv)~\cite{ogb} to conduct the experiments throughout the paper. The statistics of the datasets are provided in Table~\ref{table:datasets}. We give their detailed descriptions are as follows:

\begin{itemize}
\item
\textbf{Cora}, \textbf{CiteSeer}, and \textbf{PubMed}\footnote{\url{https://github.com/kimiyoung/planetoid/raw/master/data}}~\cite{cora} are three academic networks where nodes represent papers and edges represent citation relations. Each node in Cora and CiteSeer is described by a 0/1-valued word vector indicating the absence/presence of the corresponding word from the dictionary, while each node in PubMed is described by a TF/IDF weighted word vector from the dictionary. The nodes are categorized by their related research area for the three datasets.
\item
\textbf{Amazon-Photo} and \textbf{Amazon-Computers}\footnote{\url{https://github.com/shchur/gnn-benchmark/raw/master/data/npz}}~\cite{pitfall} are two co-purchase networks constructed from Amazon where nodes represent products and edges represent co-purchase relations.
Each node is described by a raw bag-of-words feature encoding product reviews, and is labeled with its category.
\item
\textbf{ogbn-arXiv}\footnote{\url{https://ogb.stanford.edu/docs/nodeprop/\#ogbn-arxiv}}~\cite{ogb} is a citation network between all Computer Science arXiv papers indexed by Microsoft academic graph~\cite{mag}, where nodes represent papers and edges represent citation relations. Each node is described by a 128-dimensional feature vector obtained by averaging the skip-gram word embeddings in its title and abstract. The nodes are categorized by their related research area.
\end{itemize}

\subsubsection*{\textbf{Baselines}}
We compare HomoGCL with a variety of baselines, including unsupervised methods Node2Vec~\cite{node2vec} and DeepWalk~\cite{deepwalk}, supervised methods GCN~\cite{gcn}, GAT~\cite{gat}, and GraphSAGE~\cite{graphsage}, graph autoencoders GAE and VGAE~\cite{gae}, graph contrastive learning methods including DGI~\cite{dgi}, HDI~\cite{hdmi}, GMI~\cite{gmi}, InfoGCL~\cite{infogcl}, MVGRL~\cite{mvgrl}, G-BT~\cite{gbt}, BGRL~\cite{bgrl}, AFGRL~\cite{afgcl}, CCA-SSG~\cite{cca}, COSTA~\cite{costa}, GRACE~\cite{grace}, GCA~\cite{gca}, ProGCL~\cite{progcl}, ARIEL~\cite{ariel}, and gCooL~\cite{gcool}. We also report the performance obtained using an MLP classifier on raw node features. 
The detailed description of the baselines could be found in Appendix~\ref{app:baseline}.

\begin{table*}[!t]
\centering
\begin{center}
\caption{Node classification results (accuracy(\%) $\pm$std) for 5 runs on five real-world datasets. The best results are highlighted in \textbf{boldface}. $\mathbf{X}$, $\mathbf{A}$, and $\mathbf{Y}$ correspond to node features, graph adjacency matrix, and node labels respectively. ``\green{$\uparrow$}'' and ``\red{$\downarrow$}'' refer to performance improvement
and drop compared with the same GRACE base model.}\label{table:classification}
\scalebox{1.0}{
\begin{threeparttable}
\begin{tabular}{lccccccc}
\toprule[1pt]
{Model} & {Training Data} & {Cora} & {CiteSeer} & {PubMed} & {Photo} &{Computer}\\
\midrule[0.5pt]
Raw features & $\mathbf{X}, \mathbf{Y}$ & 47.7$\pm$0.4 & 46.5$\pm$0.4 & 71.4$\pm$0.2           & 72.27$\pm$0.00 & 73.81$\pm$0.00 \\
DeepWalk & $\mathbf{A}$ & 70.7$\pm$0.6 & 51.4$\pm$0.5 & 74.3$\pm$0.9 & 89.44$\pm$0.11 & 85.68$\pm$0.06 \\
Node2Vec & $\mathbf{A}$ & 70.1$\pm$0.4 & 49.8$\pm$0.3 & 69.8$\pm$0.7           & 87.76$\pm$0.10 & 84.39$\pm$0.08 \\
GCN & $\mathbf{X, A, Y}$ & 81.5$\pm$0.4 & 70.2$\pm$0.4 & 79.0$\pm$0.2 & 92.42$\pm$0.22 & 86.51$\pm$0.54 \\
GAT & $\mathbf{X, A, Y}$ & 83.0$\pm$0.7 & 72.5$\pm$0.7 & 79.0$\pm$0.3 & 92.56$\pm$0.35 & 86.93$\pm$0.29 \\
\midrule[0.5pt]
GAE & $\mathbf{X, A}$ & 71.5$\pm$0.4 & 65.8$\pm$0.4 & 72.1$\pm$0.5 & 91.62$\pm$0.13 & 85.27$\pm$0.19 \\
VGAE & $\mathbf{X, A}$ & 73.0$\pm$0.3 & 68.3$\pm$0.4 & 75.8$\pm$0.2 & 92.20$\pm$0.11 & 86.37$\pm$0.21 \\
DGI & $\mathbf{X, A}$ & 82.3$\pm$0.6 & 71.8$\pm$0.7 & 76.8$\pm$0.6 & 91.61$\pm$0.22 & 83.95$\pm$0.47 \\
GMI & $\mathbf{X, A}$ & 83.0$\pm$0.3 & 72.4$\pm$0.1 & 79.9$\pm$0.2 & 90.68$\pm$0.17 & 82.21$\pm$0.31 \\
InfoGCL & $\mathbf{X, A}$ & 83.5$\pm$0.3 & \textbf{73.5}$\pm$0.4 & 79.1$\pm$0.2 &-&-\\
MVGRL & $\mathbf{X, A}$ & 83.5$\pm$0.4 & 73.3$\pm$0.5 & 80.1$\pm$0.7 & 91.74$\pm$0.07 & 87.52$\pm$0.11 \\
BGRL & $\mathbf{X, A}$ & 82.7$\pm$0.6 & 71.1$\pm$0.8 & 79.6$\pm$0.5 & 92.80$\pm$0.08 & 88.23$\pm$0.11 \\
AFGRL & $\mathbf{X, A}$ & 79.8$\pm$0.2 & 69.4$\pm$0.2 & 80.0$\pm$0.1 & 92.71$\pm$0.23 & 88.12$\pm$0.27 \\
COSTA & $\mathbf{X, A}$ & 82.2$\pm$0.2 & 70.7$\pm$0.5 & 80.4$\pm$0.3 & 92.43$\pm$0.38 & 88.37$\pm$0.22 \\
CCA-SSG & $\mathbf{X, A}$ & 84.0$\pm$0.4 & 73.1$\pm$0.3 & 81.0$\pm$0.4 & 92.84$\pm$0.18 & 88.27$\pm$0.32 \\
\midrule[0.5pt]
GRACE & $\mathbf{X, A}$ & 81.5$\pm$0.3 & 70.6$\pm$0.5 & 80.2$\pm$0.3 & 92.15$\pm$0.24 & 86.25$\pm$0.25 \\
GCA & $\mathbf{X, A}$ & 81.4$\pm$0.3(\red{$\downarrow$}0.1) & 70.4$\pm$0.4(\red{$\downarrow$}0.2) & 80.7$\pm$0.5(\green{$\uparrow$}0.5) & 92.53$\pm$0.16(\green{$\uparrow$}0.38) & 87.80$\pm$0.23(\green{$\uparrow$}1.55) \\
ProGCL & $\mathbf{X, A}$ & 81.2$\pm$0.4(\red{$\downarrow$}0.3) & 69.8$\pm$0.5(\red{$\downarrow$}0.8) & 79.2$\pm$0.2(\red{$\downarrow$}1.0) & 92.39$\pm$0.11(\green{$\uparrow$}0.24) & 87.43$\pm$0.21(\green{$\uparrow$}1.18) \\
ARIEL & $\mathbf{X, A}$ & 83.0$\pm$1.3(\green{$\uparrow$}1.5) & 71.1$\pm$0.9(\green{$\uparrow$}0.5) & 74.2$\pm$0.8(\red{$\downarrow$}6.0) & 91.80$\pm$0.24(\red{$\downarrow$}0.35) & 87.07$\pm$0.33(\green{$\uparrow$}0.82) \\
\rowcolor{light-gray}\textbf{HomoGCL} & $\mathbf{X, A}$ & \textbf{84.5}$\pm$0.5(\green{$\uparrow$}3.0) & 72.3$\pm$0.7(\green{$\uparrow$}1.7) & \textbf{81.1}$\pm$0.3(\green{$\uparrow$}0.9) & \textbf{92.92}$\pm$0.18(\green{$\uparrow$}0.77) & \textbf{88.46}$\pm$0.20(\green{$\uparrow$}2.21) \\
\bottomrule[1pt]
\end{tabular}
\begin{tablenotes}
\item[1] The results not reported are due to unavailable code.
\end{tablenotes}
\end{threeparttable}
}
\end{center}
\end{table*}

\subsubsection*{\textbf{Configurations and Evaluation protocol}}
Following previous work~\cite{gmi}, each model is firstly trained in an unsupervised manner using the entire graph, after which the learned embeddings are utilized for downstream tasks. We use the Adam optimizer for both the self-supervised GCL training and the evaluation stages. The graph encoder $f_{\Theta}$ is specified as a standard two-layer GCN model by default for all the datasets.
For the node classification task, we train a simple $L_2$-regularized one-layer linear classifier. For Cora, CiteSeer, and PubMed, we apply public splits~\cite{cora} to split them into training/validation/test sets, where only 20 nodes per class are available during training. For Photo and Computer, we randomly split them into training/validation/testing sets with proportions 10\%/10\%/80\% respectively following~\cite{gca}, since there are no publicly accessible splits. We train the model for five runs and report the performance in terms of accuracy.
For the node clustering task, we train a $k$-means model on the learned embeddings for 10 times, where the number of clusters is set to the number of classes for each dataset. We measure the clustering performance in terms of two prevalent metrics Normalized Mutual Information (NMI) score: $\mathrm{NMI}=2I(\mathbf{\hat{Y}};\mathbf{Y})/[H(\mathbf{\hat{Y}})+H(\mathbf{Y})]$, where $\mathbf{\hat{Y}}$ and $\mathbf{Y}$ being the predicted cluster indexes and class labels respectively, $I(\cdot)$ being the mutual information, and $H(\cdot)$ being the entropy; and Adjusted Rand Index (ARI): $\mathrm{ARI}=\mathrm{RI}-\mathbb{E}[\mathrm{RI}]/(\max\{\mathrm{RI}\}-\mathbb{E}[\mathrm{RI}])$, where $\mathrm{RI}$ being the Rand Index~\cite{rand}.

\subsection{Node Classification (RQ1)}

We implement HomoGCL based on GRACE.
The experimental results of node classification on five datasets are shown in Table~\ref{table:classification}, from which we can see that HomoGCL outperforms all self-supervised baselines or even the supervised ones, over the five datasets except on CiteSeer, which we attribute to its relatively low homophily. We can also observe that GRACE-based methods GCA, ProGCL, and ARIEL cannot bring consistent improvements over GRACE. In contrast, HomoGCL can always yield significant improvements over GRACE, especially on Cora with a 3\% gain.

We also find CCA-SSG a solid baseline, which can achieve runner-up performance on these five datasets. It is noteworthy that CCA-SSG adopts simple edge dropping and feature masking as augmentation (like HomoGCL), while the performances of baselines with elaborated augmentation (MVGRL, COSTA, GCA, ARIEL) vary from datasets. It indicates that data augmentation in GCL might be overemphasized. 

\begin{table}[!t]
\centering
\begin{center}
\caption{Node clustering results in terms of NMI and ARI on Photo and Computer datasets, where HomoGCL is implemented based on GRACE. $\Delta_x=0.01x$ is used to denote the standard deviation.}\label{table:clustering}
\scalebox{1.00}{
\begin{tabular}{l|cc|cc}
\toprule
Dataset & \multicolumn{2}{|c|}{Photo} & \multicolumn{2}{|c}{Computer}\\
\midrule
Metric & NMI & ARI & NMI & ARI \\
\midrule
GAE & 0.616$\pm\Delta_1$ & 0.494$\pm\Delta_1$ & 0.441$\pm\Delta_0$ & 0.258$\pm\Delta_0$ \\
VGAE & 0.530$\pm\Delta_4$ & 0.373$\pm\Delta_4$ & 0.423$\pm\Delta_0$ & 0.238$\pm\Delta_0$ \\
DGI & 0.376$\pm\Delta_3$ & 0.264$\pm\Delta_3$ & 0.318$\pm\Delta_2$ & 0.165$\pm\Delta_2$ \\
HDI & 0.429$\pm\Delta_1$ & 0.307$\pm\Delta_1$ & 0.347$\pm\Delta_1$ & 0.216$\pm\Delta_6$ \\
MVGRL & 0.344$\pm\Delta_4$ & 0.239$\pm\Delta_4$ & 0.244$\pm\Delta_0$ & 0.141$\pm\Delta_0$ \\
BGRL & 0.668$\pm\Delta_3$ & 0.547$\pm\Delta_4$ & 0.484$\pm\Delta_0$ & 0.295$\pm\Delta_0$ \\
AFGRL & 0.618$\pm\Delta_1$ & 0.497$\pm\Delta_3$ & 0.478$\pm\Delta_3$ & 0.334$\pm\Delta_4$ \\
GCA & 0.614$\pm\Delta_0$ & 0.494$\pm\Delta_0$ & 0.426$\pm\Delta_0$ & 0.246$\pm\Delta_0$ \\
gCooL & 0.632$\pm\Delta_0$ & 0.524$\pm\Delta_0$ & 0.474$\pm\Delta_2$ & 0.277$\pm\Delta_2$ \\
\rowcolor{light-gray}\textbf{HomoGCL} & \textbf{0.671}$\pm\Delta_2$ & \textbf{0.587}$\pm\Delta_2$ & \textbf{0.534}$\pm\Delta_0$ & \textbf{0.396}$\pm\Delta_0$ \\
\bottomrule
\end{tabular}
}
\end{center}
\end{table}

Finally, other methods which leverage graph homophily implicitly (AFGRL, ProGCL) do not perform as we expected. We attribute the non-ideal performance of AFGRL to the fact that it does not apply data augmentation, which might limit its representation ability. For ProGCL, since the model focuses on negative samples by alleviating false negative cases, it is not as effective as HomoGCL to directly expand positive ones.

\subsection{Node Clustering (RQ1)}\label{sec:cluster}

We also evaluate the node clustering performance on Photo and Computer datasets in this section. HomoGCL is also implemented based on GRACE. As shown in Table~\ref{table:clustering}, HomoGCL generally outperforms other methods by a large margin on both metrics for the two datasets. We attribute the performance to that by enlarging the connection density between node pairs far away from the estimated decision boundaries, HomoGCL can naturally acquire compact intra-cluster bonds, which directly benefits clustering. It is also validated by the visualization experiment, which will be discussed in Section~\ref{sec:visual}.

\begin{table}[!t]
\centering
\begin{center}
\caption{The performance of HomoGCL for boosting negative sample-free BGRL in terms of accuracy(\%).}\label{table:bgrl}
\scalebox{1.00}{
\begin{tabular}{l|ccc}
\toprule
Model & PubMed & Photo & Computer \\
\midrule
BGRL & 79.6 & 92.80 & 88.23 \\
\rowcolor{light-gray}\quad+\textbf{HomoGCL} & 80.8(\green{$\uparrow$}1.2) & 93.53(\green{$\uparrow$}0.73) & 90.01(\green{$\uparrow$}1.79) \\
\bottomrule
\end{tabular}
}
\end{center}
\end{table}

\subsection{Improving Various GCL Methods (RQ1)}\label{sec:plusbgrl}

As a patch to expand positive pairs, it is feasible to combine HomoGCL with other GCL methods, even negative-free ones like BGRL~\cite{bgrl}, which adapts BYOL~\cite{byol} in computer vision to GCL to free contrastive learning from numerous negative samples.

\subsubsection*{\textbf{Sketch of BGRL}} 
An online encoder $f_{\xi}$ and a target encoder $f_{\phi}$ are leveraged to encoder two graph views $\mathcal{G}_1$ and $\mathcal{G}_2$ respectively, after which we can get $\tilde{\mathbf{H}}_1$ and $\tilde{\mathbf{H}}_2$. An additional predictor $p_{\xi}$ is applied to $\tilde{\mathbf{H}}_1$ and we can get $\tilde{\mathbf{Z}}_1$. The loss function is defined as 
\begin{equation}
\ell_1=\frac{1}{N} \sum_{i=1}^{N}\theta^{\prime}\left(\tilde{\mathbf{Z}}_{(1, i)}, \tilde{\mathbf{H}}_{(2, i)}\right),
\end{equation} 
where $\theta^{\prime}(\cdot,\cdot)$ is a similarity function. A symmetric loss $\tilde{\ell}_1$ is obtained by feeding $\mathcal{G}_1$ into the target encoder and $\mathcal{G}_2$ into the online encoder, and the final objective
is $\mathcal{L}_1=\ell_1+\tilde{\ell}_1$. 
For online encoder $f_{\xi}$, its parameters are updated via stochastic gradient descent, while the parameters of the target encoder $f_{\phi}$ are updated via exponential moving average (EMA) of $\theta$ as $\phi \leftarrow \tau \phi+(1-\tau) \xi$.

To combine BGRL with HomoGCL, we feed the initial graph $\mathcal{G}$ to the online encoder and get $\mathbf{H}$. Then we get the assignment matrix $\mathbf{R}$ via Eq.~\eqref{eq:gmm} and the saliency $\mathbf{S}$ via multiplication. The expanded loss can thus be defined as 
\begin{equation}
\ell_2=\frac{1}{|\mathcal{E}|}\sum_{(v_i,v_j)\in\mathcal{E}}\theta^{\prime}\left(\tilde{\mathbf{Z}}_{(1, i)},\tilde{\mathbf{H}}_{(2, j)}\right)\mathbf{S}_{ij}.
\end{equation}
We can also get a symmetric loss $\tilde{\ell_2}$, and the entire expanded loss is $\mathcal{L}_2=\ell_2+\tilde{\ell}_2$. Together with the homophily loss $\mathcal{L}_{homo}$ defined in Eq.~\eqref{eq:homoloss}, we can get the overall loss for BGRL-based HomoGCL as
\begin{equation}
\mathcal{J}=\mathcal{L}_1+\alpha\mathcal{L}_{homo}+\beta\mathcal{L}_2,
\end{equation}
where $\beta$, $\alpha$ are two hyperparameters.

We evaluate the performance of BGRL+HomoGCL on PubMed, Photo, and Computer. 
The results are shown in Table~\ref{table:bgrl}, from which we can see that HomoGCL brings consistent improvements over the BGRL base. It verifies that HomoGCL is model-agnostic and can be applied to GCL models in a plug-and-play way to boost their performance. Moreover, it is interesting to see that the performances on Photo and Computer even surpass GRACE+HomoGCL as we reported in Table~\ref{table:classification}, which shows the potential of HomoGCL to further boost the performance of existing GCL methods.

\begin{table}[!t]
\centering
\begin{center}
\caption{Node classification results on ogbn-arXiv dataset (accuracy(\%) $\pm$std). The results of baselines are quoted from published reports. OOM indicates out-of-memory.
}\label{table:ogbn}
\scalebox{1}{
\begin{tabular}{lcc}
\toprule
Model & Validation & Test \\
\midrule
MLP & 57.65$\pm$0.12 & 55.50$\pm$0.23 \\
node2vec & 71.29$\pm$0.13 & 70.07$\pm$0.13 \\
GCN & 73.00$\pm$0.17 & 71.74$\pm$0.29 \\
GraphSAGE & 72.77$\pm$0.16 & 71.49$\pm$0.27 \\
\midrule
Random-Init & 69.90$\pm$0.11 & 68.94$\pm$0.15 \\
DGI & 71.26$\pm$0.11 & 70.34$\pm$0.16 \\
G-BT & 71.16$\pm$0.14 & 70.12$\pm$0.18 \\
GRACE full-graph & OOM & OOM \\
GRACE-Subsampling ($k$=2) & 60.49$\pm$3.72 & 60.24$\pm$4.06 \\
GRACE-Subsampling ($k$=8) & 71.30$\pm$0.17 & 70.33$\pm$0.18 \\
GRACE-Subsampling ($k$=2048) & 72.61$\pm$0.15 & 71.51$\pm$0.11 \\
ProGCL & 72.45$\pm$0.21 & 72.18$\pm$0.09 \\
\midrule
BGRL & 72.53$\pm$0.09 & 71.64$\pm$0.12 \\
\rowcolor{light-gray}\textbf{HomoGCL} & \textbf{72.85}$\pm$0.10 & \textbf{72.22}$\pm$0.15 \\
\bottomrule
\end{tabular}
}
\vspace{-0.2cm}
\end{center}
\end{table}

\subsection{Results on Large-Scale Dataset (RQ1)}

We also conduct an experiment on a large-scale dataset arXiv.
As the dataset is split based on the publication dates of the papers, i.e., the training set is papers published until 2017, the validation set is papers published in 2018, and the test set is papers published since 2019, we report the classification accuracy on both the validation and the test sets, which is a convention for this task. We extend the backbone GNN encoder to 3 GCN layers, as suggested in~\cite{bgrl, progcl}. The results are shown in Table~\ref{table:ogbn}.

Since GRACE treats all other nodes as negative samples, scaling GRACE to the large-scale dataset suffers from the OOM issue. Subsampling $k$ nodes randomly across the graph as negative samples for each node is a feasible countermeasure~\cite{bgrl}, but it is sensitive to the negative size $k$. BGRL, on the other hand, is free from negative samples, which makes it scalable by design, and it shows a great tradeoff between performance and complexity. Since the space complexity of HomoGCL mainly depends on the performance of the base model as discussed in Section~\ref{sec:complexity}, we implement HomoGCL based on BGRL as described in Section~\ref{sec:plusbgrl}. We can see that BGRL+HomoGCL boosts the performance of BGRL on both validation and test sets. Furthermore, it outperforms all other compared baselines, which shows its effectiveness and efficiency.

\subsection{Case Study (RQ2)}

\begin{figure}[!t]
\centering
\includegraphics[width=\linewidth]{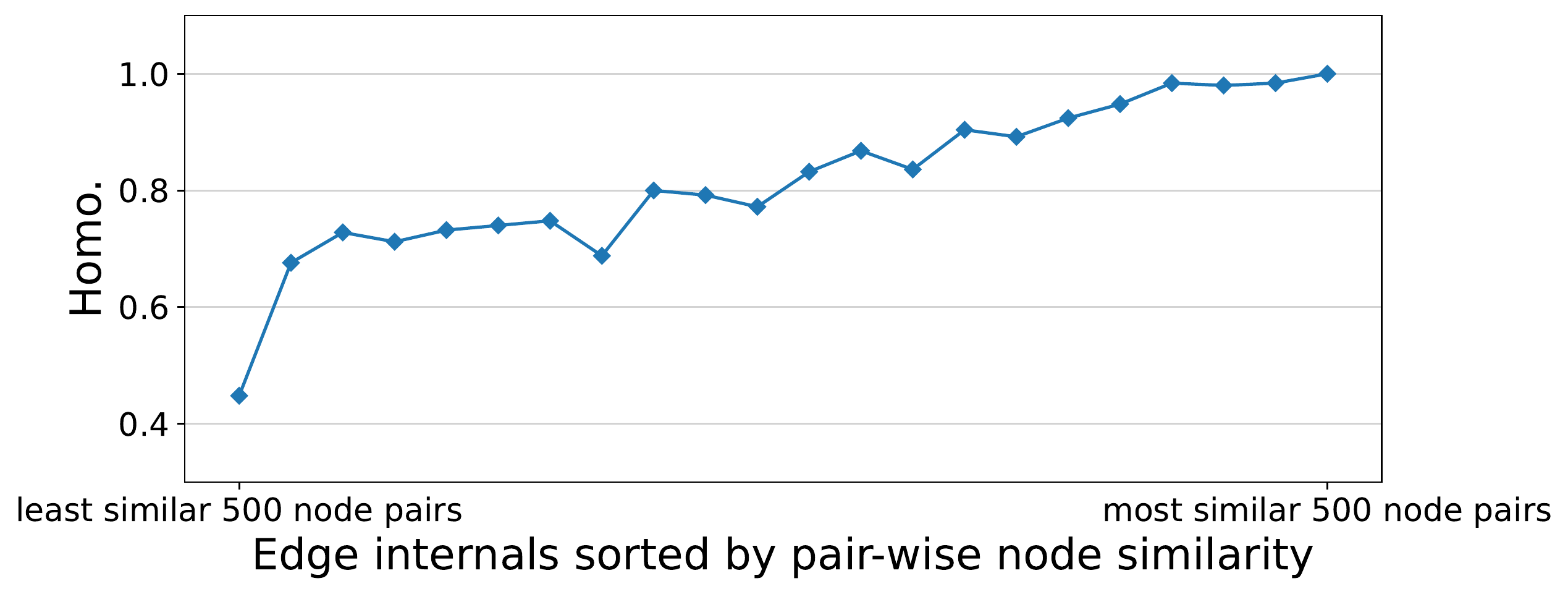}
\centering
\caption{Case study on Cora. The saliency $\mathbf{S}$ can effectively estimate the probability of neighbor nodes being positive as more salient edges (more similar node pairs) tend to have larger homophily.}
\label{fig:case}
\vspace{-0.2cm}
\end{figure}

To obtain an in-depth understanding for the mechanism of the saliency $\mathbf{S}$ in distinguishing the importance of neighbor nodes being positive, we conduct this case study. Specifically, we first sort all edges according to the learned saliency $\mathbf{S}$, then divide them into intervals of size 500 to calculate the homophily in each interval. From Figure~\ref{fig:case} we can see that the saliency can estimate the probability of neighbor nodes being positive properly as more similar node pairs in $\mathbf{S}$ tend to have larger homophily, which validates the effectiveness of leveraging saliency in HomoGCL.

\subsection{Hyperparameter Analysis (RQ3)}

In Figure~\ref{fig:hyper}, we conduct a hyperparameter analysis on the number of clusters and the weight coefficient $\alpha$ in Eq.~\eqref{eq:loss} on the Cora dataset. From the figures we can see that the performance is stable w.r.t. the cluster number. We attribute this to the soft class assignments used in identifying the decision boundary, as the pairwise saliency is mainly affected by the relative distance between node pairs, which is less sensitive to the number of clusters. The performance is also stable when $\alpha$ is below 10 (i.e., contrastive loss in Eq.~\eqref{eq:infonceloss} and homophily loss in Eq.~\eqref{eq:homoloss} are on the same order of magnitude). This shows that HomoGCL is parameter-insensitive, thus facilitating it to be combined with other GCL methods in a plug-and-play way flexibly. In practice, we tune the number of clusters in $\{5,10,15,20,25,30\}$ and simply assign $\alpha$ to 1 to balance the two losses without tuning.

\begin{figure}[!t]
\centering
\subfigure
{
\centering
\includegraphics[width=0.47\linewidth]{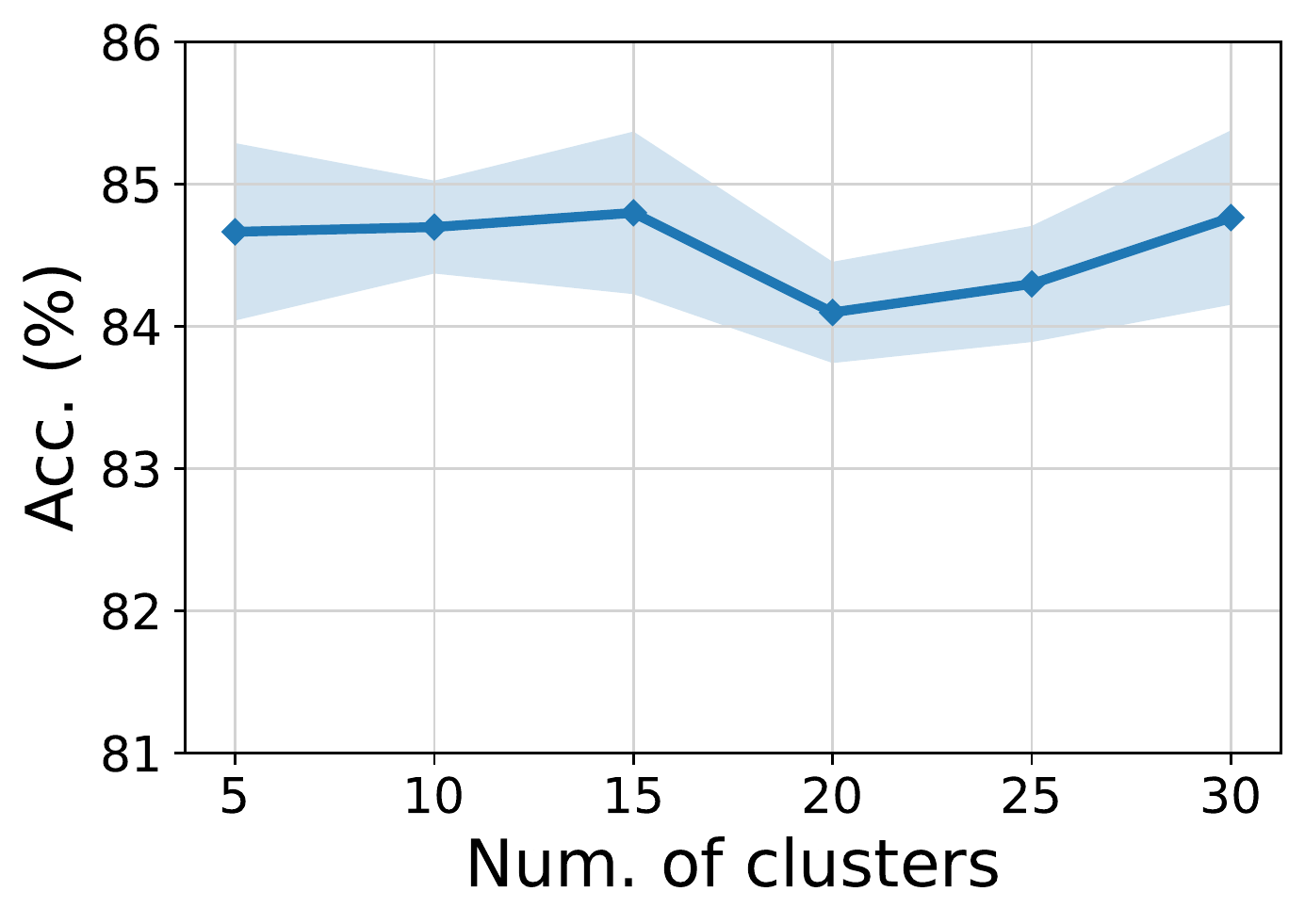}
}
\subfigure
{
\centering
\includegraphics[width=0.47\linewidth]{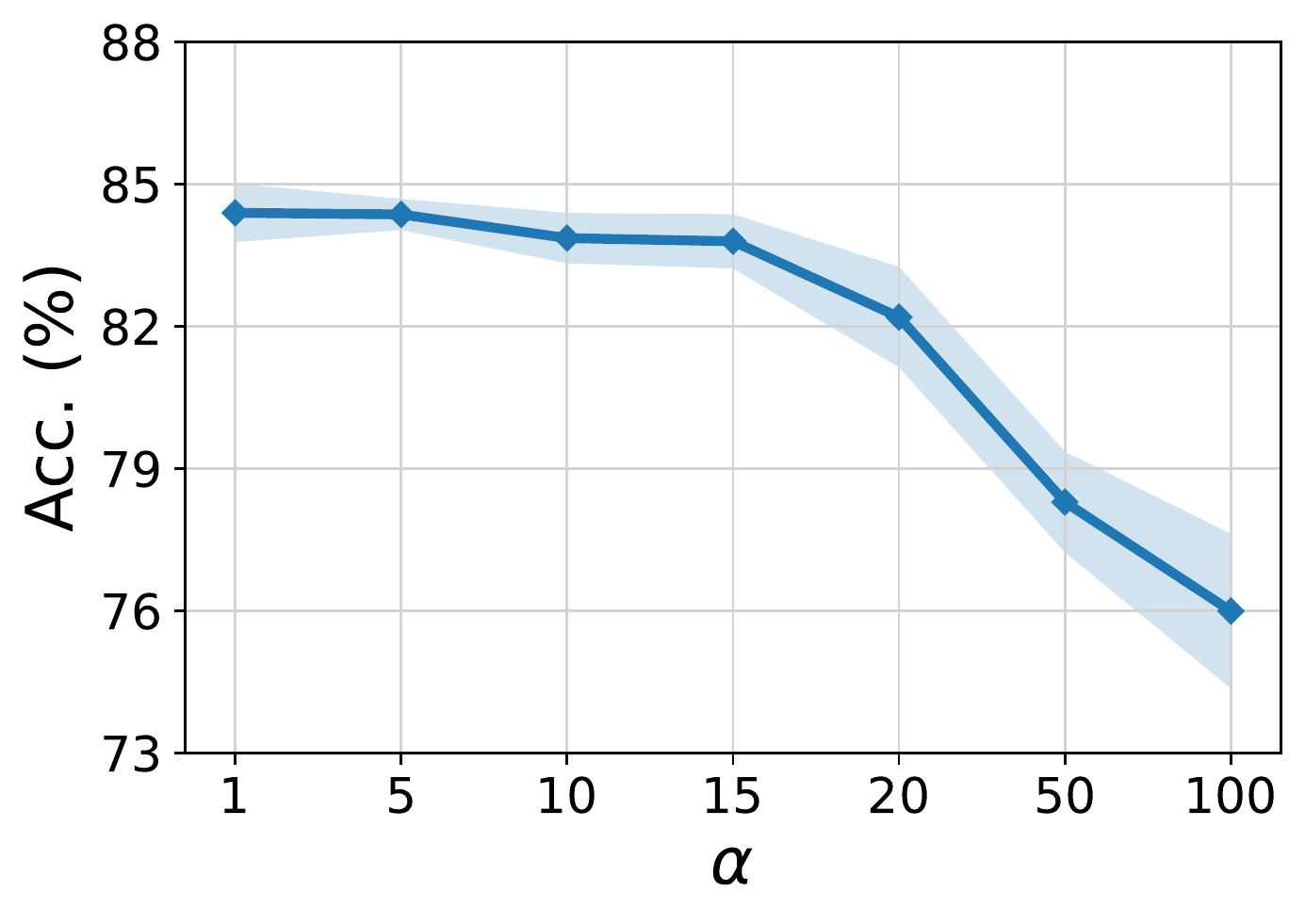}
}
\centering
\vspace{-0.2cm}
\caption{Hyperparameter analysis on the number of clusters and weight coefficient $\alpha$ on Cora.}
\label{fig:hyper}
\vspace{-0.3cm}
\end{figure}

\begin{figure}[!t]
\centering
\subfigure[BGRL]
{
\centering
\includegraphics[width=0.47\linewidth]{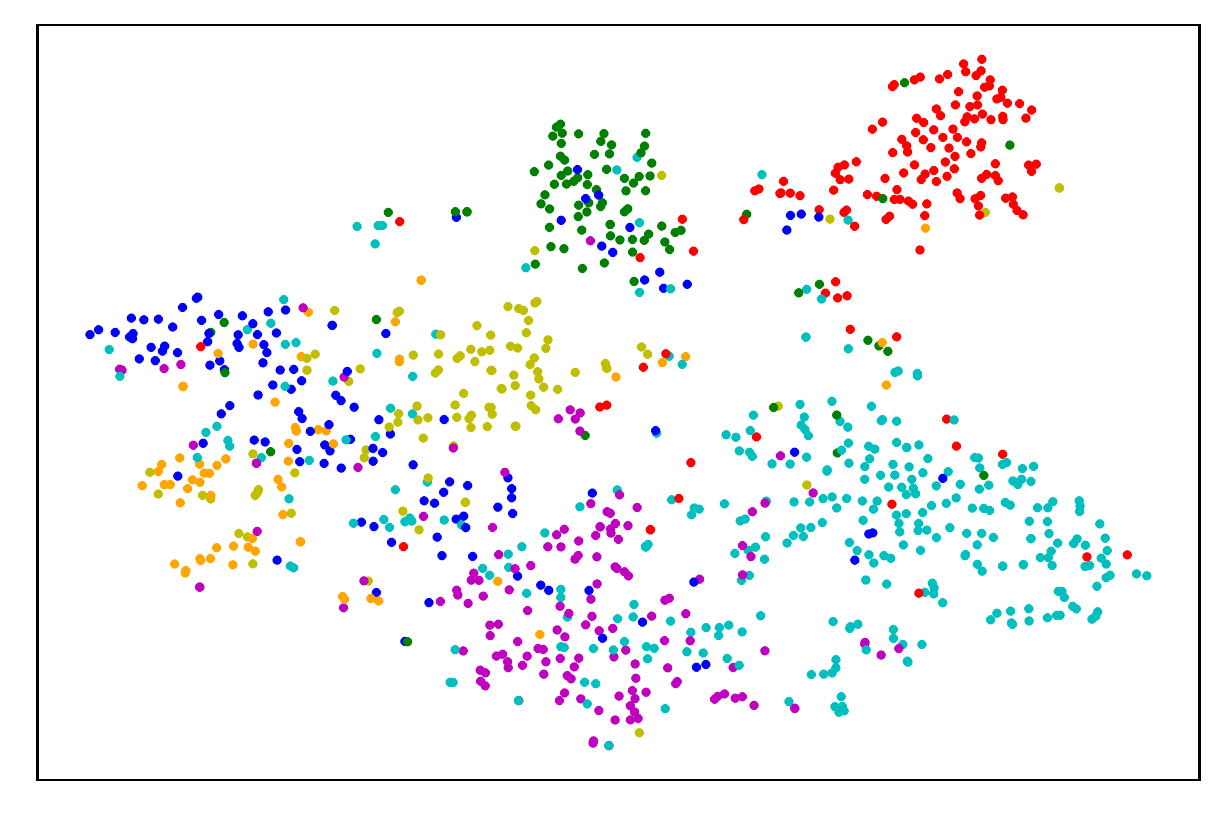}
}
\subfigure[GRACE]
{
\centering
\includegraphics[width=0.47\linewidth]{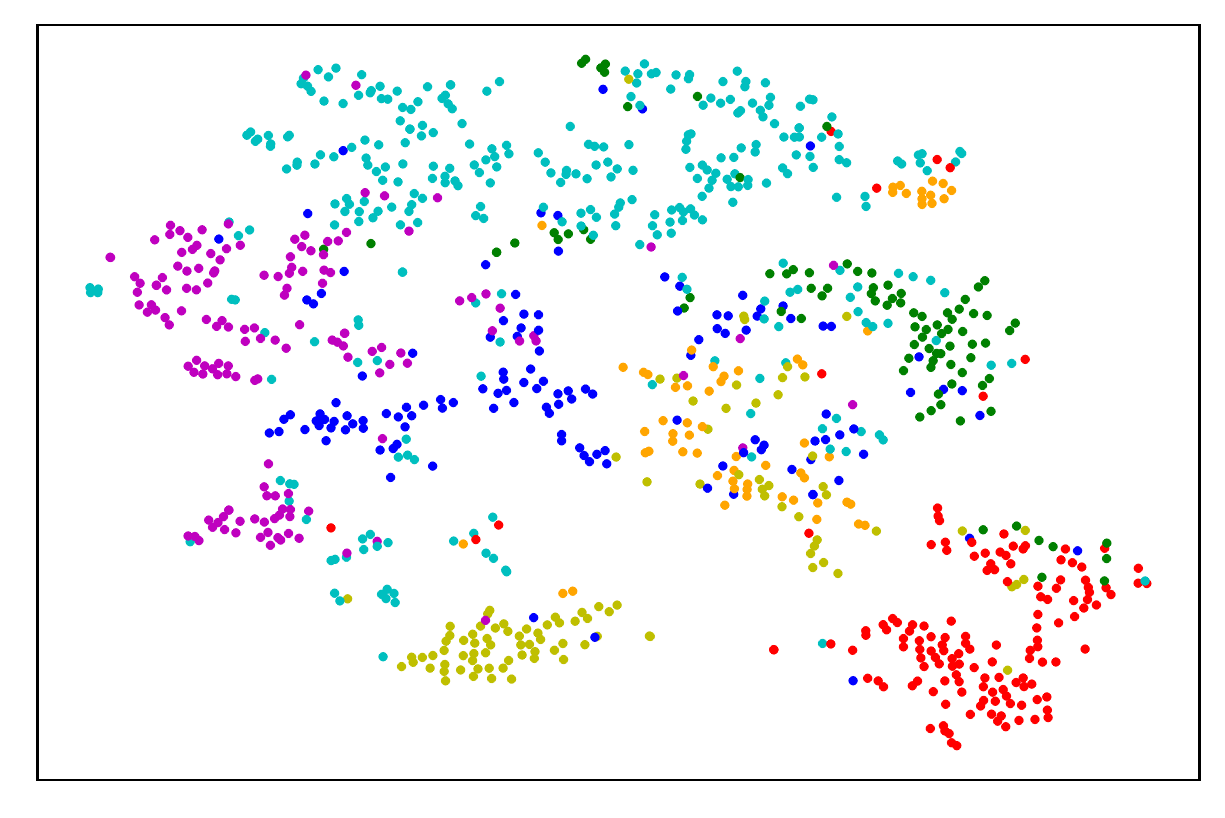}
}
\subfigure[BGRL + HomoGCL]
{
\centering
\includegraphics[width=0.47\linewidth]{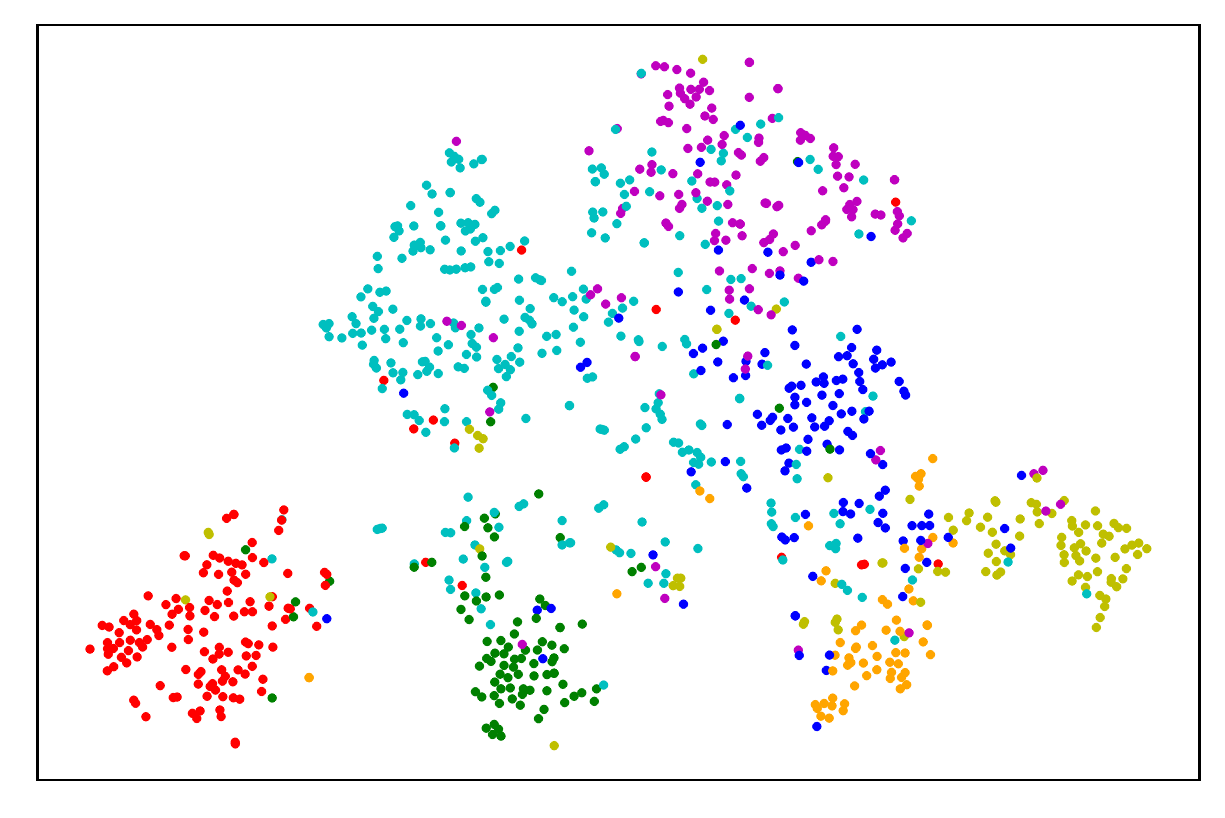}
}
\subfigure[GRACE + HomoGCL]
{
\centering
\includegraphics[width=0.47\linewidth]{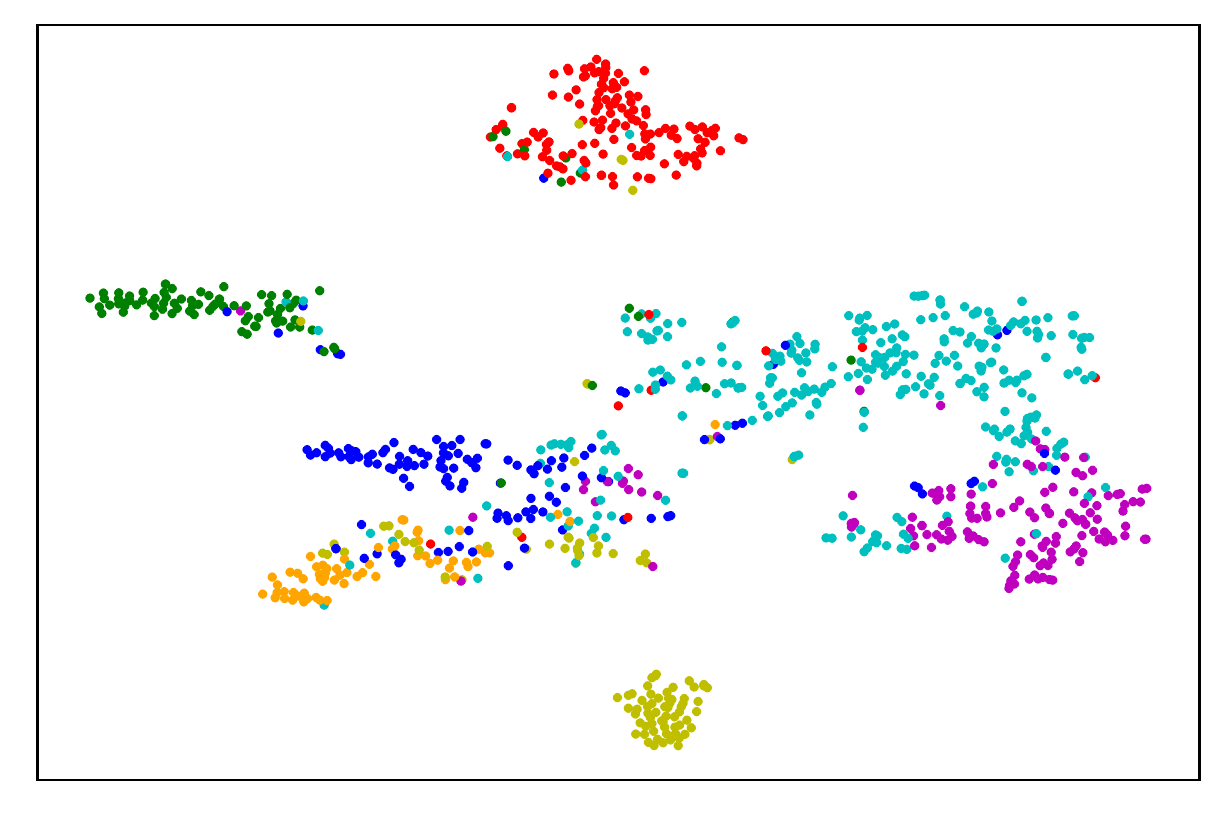}
}
\centering
\vspace{-0.2cm}
\caption{Visualization of node embeddings on Cora via t-SNE. Each node is colored by its label.}\label{fig:tsne}
\vspace{-0.3cm}
\end{figure}

\subsection{Visualization (RQ4)}\label{sec:visual}

In addition to quantitative analysis, we also visualize the embeddings learned by BGRL in Figure~\ref{fig:tsne}(a), GRACE in Figure~\ref{fig:tsne}(b), BGRL+ HomoGCL in Figure~\ref{fig:tsne}(c), and GRACE+HomoGCL in Figure~\ref{fig:tsne}(d) on the Cora dataset using t-SNE~\cite{tsne}. Here, each point represents a node and is colored by its label. We observe that the embeddings learned by ``+HomoGCL'' counterparts generally possess clearer class boundaries and compact intra-class structures, which shows the effectiveness of HomoGCL intuitively. This observation aligns with the remarkable node clustering performance reported in Section~\ref{sec:cluster}, which shows the superiority of HomoGCL.

\section{Conclusions}\label{sec:conc}

In this paper, we investigate why graph contrastive learning can perform decently when data augmentation is not leveraged and argue that graph homophily plays a key role in GCL. We thus devise HomoGCL to directly leverage homophily by estimating the probability of neighbor nodes being positive via Gaussian Mixture Model.
Furthermore, HomoGCL is model-agnostic and thus can be easily combined with existing GCL methods in a plug-and-play way to further boost their performances with a theoretical foundation.
Extensive experiments show that HomoGCL can consistently outperform state-of-the-art GCL methods in node classification and node clustering tasks on six benchmark datasets. 

\begin{acks}
The authors would like to thank Ying Sun from The Hong Kong University of Science and Technology (Guangzhou) for her insightful discussion.
This work was supported by NSFC (62276277), Guangdong Basic  Applied Basic Research Foundation (2022B1515120059), and the Foshan HKUST Projects (FSUST21-FYTRI01A, FSUST21-FYTRI02A).
Chang-Dong Wang and Hui Xiong are the corresponding authors.
\end{acks}

\bibliographystyle{ACM-Reference-Format}
\balance
\bibliography{sample-base}

\clearpage
\appendix

\section*{Appendix}

\setcounter{equation}{0}
\renewcommand\theequation{A.\arabic{equation}}

\section{Main Notations}\label{app:notation}

The main notions used the paper are elaborated in Table~\ref{table:notationtable}.

\begin{table}[!h]
\caption{Summary of the main notations used in the paper.}\label{table:notationtable}
\begin{center}
\begin{tabular}{ll}
\toprule[1pt]
{Symbol} &{Description}\\
\midrule[0.5pt]
$\mathbf{M}$, $\boldsymbol{m}_i$ & matrix, $i$-th row of the matrix \\
$\mathcal{G}$ & input graph \\
$\mathcal{V}$ & node set of the graph \\
$\mathcal{E}$ & edge set of the graph \\
$N=|\mathcal{V}|$ & number of nodes in the graph \\
$\mathbf{X}\in\mathbb{R}^{N\times d}$ & node features of the graph \\
$\mathbf{A}\in\mathbb{R}^{N\times N}$ & adjacency matrix of the graph \\
$\mathbf{Y}\in\mathbb{R}^{N}$ & node labels of the graph \\
$t_1$, $t_2$ & graph augmentation functions \\
$\mathcal{G}_1$, $\mathcal{G}_2$ & augmented graphs \\
$\mathcal{N}_1(i)$ , $\mathcal{N}_2(i)$& node sets of $v_i$'s neighbors in $\mathcal{G}_1$, $\mathcal{G}_2$ \\
$f_{\Theta}(\mathbf{X},\mathbf{A})$ & GNN encoder with parameter $\Theta$ \\
$\mathbf{H},\mathbf{U},\mathbf{V}\in\mathbb{R}^{N\times d^{\prime}}$ & encoded embeddings of $\mathcal{G},\mathcal{G}_1,\mathcal{G}_2$ \\
$k$ & number of clusters \\
$\mathbf{C}\in\mathbb{R}^{k}$ & cluster centroids \\
$\mathbf{R}\in\mathbb{R}^{N\times k}$ & cluster assignment matrix \\
$\mathbf{S}\in\mathbb{R}^{N\times N}$ & saliency \\
$\theta(\cdot,\cdot)$ & similarity function \\
$\tau$ & temperature parameter \\
$\alpha$ & coefficient of loss functions \\
\bottomrule
\end{tabular}
\end{center}
\vspace{-0.3cm}
\end{table}

\section{Detailed Proofs}\label{app:proof}

W preasent the proof of \textsc{Theorem}~\ref{theorem:1}:
\textit{The newly proposed contrastive loss $\mathcal{L}_{cont}$ in Eq.~\eqref{eq:cont} is a stricter lower bound of MI between raw node features $\mathbf{X}$ and node embeddings $\mathbf{U}$ and $\mathbf{V}$ in two augmented views, comparing with the raw contrastive loss $\mathcal{L}$ in Eq.~\eqref{eq:grace} proposed by GRACE. Formally,}
\begin{equation}
\mathcal{L} \leq \mathcal{L}_{cont} \leq I(\mathbf{X};\mathbf{U},\mathbf{V}).
\end{equation}

\begin{proof}[Proof for the first inequality] 
We first prove $\mathcal{L} \leq \mathcal{L}_{cont}$.

Each element in the saliency satisfies $\mathbf{S}_{ij}\in[0,1]$. Note that
\begin{equation}
e^{\theta(\boldsymbol{u}_{i}, \boldsymbol{v}_{i}) / \tau}\leq 
e^{\theta(\boldsymbol{u}_i,\boldsymbol{v}_i)/ \tau}+\sum\nolimits_{j\in \mathcal{N}_{\boldsymbol{u}}(i)}e^{\theta(\boldsymbol{u}_i,\boldsymbol{u}_j)/\tau}\cdot \mathbf{S}_{ij},
\end{equation}
and 
\begin{equation}
\begin{aligned}
&e^{\theta(\boldsymbol{u}_{i}, \boldsymbol{v}_{i}) / \tau}+\sum\nolimits_{j \neq i} e^{\theta(\boldsymbol{u}_{i}, \boldsymbol{v}_{j}) / \tau}+\sum\nolimits_{j \neq i} e^{\theta(\boldsymbol{u}_{i}, \boldsymbol{u}_{j}) / \tau} \geq \\
& \qquad e^{\theta(\boldsymbol{u}_i,\boldsymbol{v}_i)/ \tau}+\sum\nolimits_{j\in \mathcal{N}_{\boldsymbol{u}}(i)}e^{\theta(\boldsymbol{u}_i,\boldsymbol{u}_j)/\tau}\cdot \mathbf{S}_{ij}\\
&+\sum\nolimits_{j \notin \{i\cup\mathcal{N}_{\boldsymbol{v}}(i)\}}e^{\theta(\boldsymbol{u}_i,\boldsymbol{v}_j)/ \tau} + \sum\nolimits_{j \notin \{i\cup\mathcal{N}_{\boldsymbol{u}}(i)\}}e^{\theta(\boldsymbol{u}_i,\boldsymbol{u}_j)/ \tau},
\end{aligned}
\end{equation}
we can thus get 
$\ell\left(\boldsymbol{u}_{i}, \boldsymbol{v}_{i}\right)\leq\ell_{cont}\left(\boldsymbol{u}_{i}, \boldsymbol{v}_{i}\right)$. 
In a similar way, we can derive
$\ell\left(\boldsymbol{v}_{i}, \boldsymbol{u}_{i}\right)\leq\ell_{cont}\left(\boldsymbol{v}_{i}, \boldsymbol{u}_{i}\right)$.
Therefore,
\begin{equation}
\begin{aligned}
\frac{1}{2 N} &\sum_{i=1}^{N}\left(\ell\left(\boldsymbol{u}_{i}, \boldsymbol{v}_{i}\right)+\ell\left(\boldsymbol{v}_{i}, \boldsymbol{u}_{i}\right)\right)\leq \\
&\frac{1}{2 N} \sum_{i=1}^{N}\left(\ell_{cont}\left(\boldsymbol{u}_{i}, \boldsymbol{v}_{i}\right)+\ell_{cont}\left(\boldsymbol{v}_{i}, \boldsymbol{u}_{i}\right)\right),
\end{aligned}
\end{equation}
i.e., $\mathcal{L} \leq \mathcal{L}_{cont}$, which concludes the proof of the first inequality.
\end{proof}

\begin{proof}[Proof for the Second inequality] 
Now we prove $\mathcal{L}_{cont} \leq I(\mathbf{X};\mathbf{U},\mathbf{V})$.

The InfoNCE~\cite{infonce, infonce2} objective is defined as
\begin{equation}
I_{\mathrm{NCE}}(\mathbf{U} ; \mathbf{V}) \triangleq \mathbb{E}\left[\frac{1}{N} \sum_{i=1}^{N} \log \frac{e^{\theta\left(\boldsymbol{u}_{i}, \boldsymbol{v}_{i}\right)}}{\frac{1}{N} \sum_{j=1}^{N} e^{\theta\left(\boldsymbol{u}_{i}, \boldsymbol{v}_{j}\right)}}\right],
\end{equation}
where the expectation is over $N$ samples from the joint distribution $\prod_{i} p\left(\boldsymbol{u}_{i}, \boldsymbol{v}_{i}\right)$. 
The proposed loss function includes two parts as
\begin{equation}
\begin{aligned}
\mathcal{L}_{cont}&=\frac{1}{2}\left(\frac{1}{N} \sum_{i=1}^{N}\ell_{cont}\left(\boldsymbol{u}_{i}, \boldsymbol{v}_{i}\right) +\frac{1}{N} \sum_{i=1}^{N}\ell_{cont}\left(\boldsymbol{v}_{i}, \boldsymbol{u}_{i}\right)\right) \\
&=\frac12\left(\mathcal{L}_{cont}^{(1)}+\mathcal{L}_{cont}^{(2)}\right).
\end{aligned}
\end{equation}
Here, the first term $\mathcal{L}_{cont}^{(1)}$ can be rewritten as
\begin{equation}
\mathcal{L}_{cont}^{(1)}= \frac{1}{N} \sum_{i=1}^{N}\ell_{cont}\left(\boldsymbol{u}_{i}, \boldsymbol{v}_{i}\right)=\mathbb{E}\left[ \frac{1}{N} \sum_{i=1}^{N}\log\frac{\mathrm{pos}}{\mathrm{pos}+\mathrm{neg}} \right]
\end{equation}
with
\begin{align}
\mathrm{pos}&=e^{\theta(\boldsymbol{u}_i,\boldsymbol{v}_i)}+\sum\nolimits_{j\in \mathcal{N}_{\boldsymbol{u}}(i)}e^{\theta(\boldsymbol{u}_i,\boldsymbol{u}_j)}\cdot \mathbf{S}_{ij},\\
\mathrm{neg}&=\sum\nolimits_{j \notin \{i\cup\mathcal{N}_{\boldsymbol{v}}(i)\}}e^{\theta(\boldsymbol{u}_i,\boldsymbol{v}_j)} + \sum\nolimits_{j \notin \{i\cup\mathcal{N}_{\boldsymbol{u}}(i)\}}e^{\theta(\boldsymbol{u}_i,\boldsymbol{u}_j)}.
\end{align}
Here, we let $\tau=1$ for convenience.
We notice that the graphs in real-world scenarios are sparse, i.e., $|i\cup\mathcal{N}_{\boldsymbol{v}}(i)| \ll N$ and $|i\cup\mathcal{N}_{\boldsymbol{u}}(i)| \ll N$ hold. Therefore, we have $\mathrm{neg}\approx \sum_{j=1}^{N} e^{\theta\left(\boldsymbol{u}_{i}, \boldsymbol{v}_{j}\right)} + \sum_{j=1}^{N} e^{\theta\left(\boldsymbol{u}_{i}, \boldsymbol{u}_{j}\right)}$, which indicates
\begin{equation}
\begin{aligned}
\mathbb{E} & \left[ \frac{1}{N} \sum_{i=1}^{N}\log \frac{\mathrm{pos}}{\mathrm{pos}+\mathrm{neg}} \right] \leq
\mathbb{E}\left[ \frac{1}{N} \sum_{i=1}^{N}\log\frac{\mathrm{pos}}{\mathrm{neg}} \right] \\
& \leq \mathbb{E}\left[ \frac{1}{N} \sum_{i=1}^{N}\log\frac{\mathrm{pos}}{\sum_{j=1}^{N} e^{\theta\left(\boldsymbol{u}_{i}, \boldsymbol{v}_{j}\right)}} \right].
\end{aligned}
\end{equation}
As $|i\cup\mathcal{N}_{\boldsymbol{v}}(i)| \ll N$ and $0\leq \mathbf{S}_{ij} \leq 1$, for a small constant $M$, we can get
\begin{equation}
\begin{aligned}
\mathrm{pos}&=e^{\theta(\boldsymbol{u}_i,\boldsymbol{v}_i)}+\sum\nolimits_{j\in \mathcal{N}_{\boldsymbol{u}}(i)}e^{\theta(\boldsymbol{u}_i,\boldsymbol{u}_j)}\cdot \mathbf{S}_{ij} \\
&\leq M\cdot  e^{\theta(\boldsymbol{u}_i,\boldsymbol{v}_i)}\leq N \cdot e^{\theta\left(\boldsymbol{u}_{i}, \boldsymbol{v}_{i}\right)}.
\end{aligned}
\end{equation}
Therefore,
\begin{equation}
\begin{aligned}
\mathbb{E}&\left[ \frac{1}{N} \sum_{i=1}^{N}\log\frac{\mathrm{pos}}{\sum_{j=1}^{N} e^{\theta\left(\boldsymbol{u}_{i}, \boldsymbol{v}_{j}\right)}} \right] \\
&\leq\mathbb{E}\left[\frac{1}{N} \sum_{i=1}^{N} \log \frac{N\cdot e^{\theta\left(\boldsymbol{u}_{i}, \boldsymbol{v}_{i}\right)}}{\sum_{j=1}^{N} e^{\theta\left(\boldsymbol{u}_{i}, \boldsymbol{v}_{j}\right)}}\right]=I_{\mathrm{NCE}}(\mathbf{U} ; \mathbf{V}).
\end{aligned}
\end{equation}
i.e., 
$\mathcal{L}_{cont}^{(1)} \leq I_{NCE}(\mathbf{U};\mathbf{V})$. 
In a similar way, we can also derive 
$\mathcal{L}_{cont}^{(2)} \leq I_{NCE}(\mathbf{V};\mathbf{U})$.
Therefore, we have
\begin{equation}
\mathcal{L}_{cont}\leq \frac{1}{2}\left(I_{NCE}(\mathbf{U};\mathbf{V})+I_{NCE}(\mathbf{V};\mathbf{U})\right).
\end{equation}
According to~\cite{infonce2}, the InfoNCE is a lower bound of MI, i.e., 
\begin{equation}
I_{NCE}(\mathbf{U};\mathbf{V})\leq I(\mathbf{U};\mathbf{V}).
\end{equation}
Therefore, we have 
\begin{equation}\label{eq:app}
\mathcal{L}_{cont}\leq \frac12\left(I(\mathbf{U};\mathbf{V})+I(\mathbf{V};\mathbf{U})\right)=I(\mathbf{U};\mathbf{V}).
\end{equation}

Following~\cite{Cover:2006ei}, for $\mathbf{X}$, $\mathbf{U}$, $\mathbf{V}$ satisfying $\mathbf{U}\leftarrow \mathbf{X}\rightarrow \mathbf{V}$, 
$I(\mathbf{U};\mathbf{V})\leq I(\mathbf{U};\mathbf{X})$
holds. 
Following~\cite{gca}, as $\mathbf{X}$, $\mathbf{U}$, $\mathbf{V}$ also satisfies $\mathbf{X}\rightarrow(\mathbf{U};\mathbf{V})\rightarrow\mathbf{U}$,
$I(\mathbf{X};\mathbf{U})\leq I(\mathbf{X};\mathbf{U},\mathbf{V})$
holds.
Combining the two inequality together with Eq.~\eqref{eq:app}, we finally have 
\begin{equation}
\mathcal{L}_{cont}\leq I(\mathbf{U};\mathbf{V})\leq I(\mathbf{U};\mathbf{X})\leq I(\mathbf{X};\mathbf{U},\mathbf{V}),
\end{equation}
which concludes the proof of the second inequality.
\end{proof}

\section{Baselines}\label{app:baseline}

In this section, we give brief introductions of the baselines used in the paper which are not described in the main paper due to the space constraint.

\begin{itemize}
\setlength{\leftskip}{-2em}
\item \textbf{DeepWalk}~\cite{deepwalk}, \textbf{Node2vec}~\cite{node2vec}: DeepWalk and Node2vec are two unsupervised random walk-based models. DeepWalk adopts skip-gram on node sequences generated by random walk, while Node2vec extends DeepWalk by considering DFS and BFS when sampling neighbor nodes. They only leverage graph topology structure while ignoring raw node features.
\item \textbf{GCN}~\cite{gcn}, \textbf{GAT}~\cite{gat}, \textbf{GraphSAGE}~\cite{graphsage}: GCN, GAT, GraphSAGE are three popular supervised GNNs. Structural information, raw node features, and node labels of the training set are leveraged.
\item\textbf{GAE/VGAE}~\cite{gae}: GAE and VGAE are graph autoencoders that learn node embeddings via vallina/variational autoencoders. Both the encoder and the decoder are implemented with graph convolutional network.
\item\textbf{DGI}~\cite{dgi}: DGI maximizes the mutual information between patch representations and corresponding high-level summaries of graphs which are derived using graph convolutional network.
\item\textbf{HDI}~\cite{hdmi} extends DGI by considering both extrinsic and intrinsic signals via high-order mutual information. Graph convolutional network is also leveraged as the encoder.
\item\textbf{GMI}~\cite{gmi}: GMI applies cross-layer node contrasting and edge contrasting. It also generalizes the idea of conventional mutual information computations to the graph domain.
\item\textbf{InfoGCL}~\cite{infogcl}: InfoGCL follows the Information Bottleneck principle to reduce the mutual information between contrastive parts while keeping task-relevant information intact at both the levels of the individual module and the entire framework.
\item\textbf{MVGRL}~\cite{mvgrl}: MVGRL maximizes the mutual information between the cross-view representations of nodes and graphs using graph diffusion.
\item\textbf{G-BT}~\cite{gbt}: G-BT utilizes a cross-correlation-based loss function instead of negative samples, which enjoys fewer hyperparameters and substantially shorter computation time.
\item\textbf{BGRL}~\cite{bgrl}: BGRL adopts asymmetrical BYOL~\cite{byol} structure to do the node-node level contrast without negative samples to avoid quadratic bottleneck.
\item\textbf{AFGRL}~\cite{afgcl}: AFGRL extends BGRL by generating an alternative view of a graph via discovering nodes that share the local structural information and the global semantics with the graph.
\item\textbf{CCA-SSG}~\cite{cca}: CCA-SSG leverages classical Canonical Correlation Analysis to construct feature-level objective which can discard augmentation-variant information and prevent degenerated solutions.
\item\textbf{COSTA}~\cite{costa}: COSTA alleviates the highly biased node embedding obtained via graph augmentation by performing feature augmentation, i.e., generates augmented features by maintaining  a good sketch of original features.
\item\textbf{GRACE}~\cite{grace}: GRACE adopts the SimCLR architecture which performs graph augmentation on the input graph and considers node-node level contrast on both inter-view and intra-view levels.
\item\textbf{GCA}~\cite{gca}: GCA extends GRACE by considering adaptive graph augmentations based on degree centrality, eigenvector centrality, and PageRank centrality.
\item\textbf{ProGCL}~\cite{progcl}: ProGCL extends GRACE by leveraging hard negative samples via Expectation Maximization to fit the observed node-level similarity distribution. We adopt the ProGCL-weight version as no synthesis of new nodes is leveraged.
\item\textbf{ARIEL}~\cite{ariel}: ARIEL extends GRACE by introducing an adversarial graph view and an information regularizer to extract informative contrastive samples within a reasonable constraint.
\item\textbf{gCooL}~\cite{gcool}: gCooL extends GRACE by jointly
learning the community partition and node representations in
an end-to-end fashion to directly leverage the community structure of a graph.
\end{itemize}

\section{Graph Augmentations}

As we mentioned in the main paper, we adopt two simple and widely used augmentation schemes, i.e., edge dropping and feature masking, as the graph augmentation. For edge dropping, we drop each edge with probability $p_e$. For feature masking, we set each raw feature as $0$ with probability $p_f$. $p_e$ and $p_f$ are set as the same for two augmented graph views, which is widely used in the literature~\cite{grace, gca, cca}. 

\section{Ablation Study}

\begin{table}[!b]
\centering
\vspace{-0.2cm}
\begin{center}
\caption{Ablation study evaluated on Cora.}\label{table:ablation}
\scalebox{1.00}{
\begin{tabular}{l|c}
\toprule
Model & Accuracy(\%) \\
\midrule
HomoGCL & 84.5$\pm$0.5 \\
\quad w/o homophily loss & 84.1$\pm$0.4 \\
\midrule
HomoGCL$_{hd}$ & 84.2$\pm$0.4 \\
\quad w/o homophily loss & 83.9$\pm$0.3 \\
\bottomrule
\end{tabular}
}
\end{center}
\end{table}

In this section, we investigate how each component of HomoGCL, including the homophily loss (Eq.~\eqref{eq:homoloss}) and the soft pair-wise node similarity (the saliency $\mathbf{S}$) contributes to the overall performance. The result is shown in Table~\ref{table:ablation}. Here, in ``w/o homophily loss'', we disable the homophily loss in (Eq.~\eqref{eq:homoloss}), and in ``HomoGCL$_{hd}$'', we add neighbor nodes as positive samples indiscriminately (i.e., hard neighbors) to leverage graph homophily directly. From the table, we have two observations:
\begin{itemize}
\setlength{\leftskip}{-2em}
\item Homophily loss is effective, as graph homophily is a distinctive inductive bias for graph data in-the-wild.
\item HomoGCL which estimates the probability of neighbors being positive performs better than HomoGCL$_{hd}$, as the soft clustering can successfully distinguish true positive samples.
\end{itemize}
The above observations show the effectiveness of the homophily loss and the saliency $\mathbf{S}$ in HomoGCL.

\end{document}